\documentclass[lettersize,journal]{IEEEtran}
\usepackage{amsmath,amsfonts}
\usepackage{algorithm}
\usepackage{array}
\usepackage[caption=false,font=normalsize,labelfont=sf,textfont=sf]{subfig}
\usepackage{textcomp}
\usepackage{stfloats}
\usepackage{url}
\usepackage{verbatim}
\usepackage{graphicx}
\usepackage{cite}
\usepackage{amsmath}
\usepackage{algorithmicx}
\usepackage{booktabs}
\usepackage{algpseudocode}
\usepackage{mathrsfs}
\usepackage{multirow}
\usepackage{color}
\usepackage{colortbl}
\usepackage{arydshln}
\usepackage{array}
\usepackage{bm}
\usepackage{rotating}
\usepackage{adjustbox}
\usepackage{color}
\usepackage{amsfonts}
\usepackage{amssymb}
\usepackage{bbding}
\usepackage{amsthm}
\usepackage{enumitem}
\usepackage{dsfont}
\newtheorem{theorem}{Theorem}
\newtheorem{lemma}{Lemma}
\newcommand{\probP}{\text{I\kern-0.15em P}}

\begin{document}

\title{The Lottery Ticket Hypothesis for Self-attention in Convolutional Neural Network$^\clubsuit$}

\author{Zhongzhan Huang, Senwei Liang$^*$, Mingfu Liang, Wei He, Haizhao Yang$^\dag$, and Liang Lin$^\dag$
        
\thanks{$^*$ Zhongzhan Huang and Senwei Liang have equal contributions.}%

\thanks{$^\dag$ Correspondence should be addressed to yang1863@purdue.edu, linliang@ieee.org.}

\thanks{Zhongzhan Huang and Liang Lin are with
 School of Computer Science and Engineering, Sun Yat-sen University, Guangzhou 510275, China (e-mail: huangzhzh23@mail2.sysu.edu.cn; linliang@ieee.org).}

\thanks{Senwei Liang and Haizhao Yang are with the Department of Mathematics, Purdue University, West Lafayette, IN, United States (e-mail: liang339@purdue.edu; yang1863@purdue.edu)}%

\thanks{Mingfu Liang is with the Department of Electrical and Computer Engineering, Northwestern University, Evanston, IL, United States (e-mail: mingfuliang2020@u.northwestern.edu).} \thanks{Wei He is with the School of Computer Science and Engineering, Nanyang Technological University, Singapore (email: wei005@e.ntu.edu.sg).}%
\thanks{$\clubsuit$ Technical report}%
}



\maketitle

\begin{abstract}
Recently many plug-and-play self-attention modules (SAMs) are proposed to enhance the model generalization by exploiting the internal information of deep convolutional neural networks~(CNNs). In general, previous works ignore where to plug in the SAMs since they connect the SAMs individually with each block of the entire CNN backbone for granted, leading to incremental computational cost and the number of parameters with the growth of network depth. However, we empirically find and verify some counterintuitive phenomena that: (a) Connecting the SAMs to all the blocks may not always bring the largest performance boost, and connecting to partial blocks would be even better; (b) Adding the SAMs to a CNN may not always bring a performance boost, and instead it may even harm the performance of the original CNN backbone.

Therefore, we articulate and demonstrate the Lottery Ticket Hypothesis for Self-attention Networks: a full self-attention network contains a subnetwork with sparse self-attention connections that can (1) accelerate inference,~(2) reduce extra parameter increment, and~(3) maintain accuracy. In addition to the empirical evidence, this hypothesis is also supported by our theoretical evidence. Furthermore, we propose a simple yet effective reinforcement-learning-based method to search the ticket, \textit{i.e}., the connection scheme that satisfies the three above-mentioned conditions. Extensive experiments on widely-used benchmark datasets and popular self-attention networks show the effectiveness of our method. Besides, our experiments illustrate that our searched ticket has the capacity of transferring to some vision tasks, \textit{e.g.}, crowd counting and segmentation.
\end{abstract}


\section{Introduction}
\label{sec:1}

\IEEEPARstart{R}{ecently}, various plug-and-play self-attention modules (SAMs) which enhance instance specificity by the interior network information~\cite{liang2020instance} are proposed to boost the generalization of convolutional neural networks~(CNNs)~\cite{hu2018squeeze,woo2018cbam,li2019spatial,huang2020dianet,cao2019gcnet,wang2018non}. The SAM is usually plugged into every block of a CNN, \textit{e.g.}, the residual block of ResNet~\cite{he2016deep}. We display the structure of a ResNet in Fig.~\ref{fig:comparsion} (a) and a full self-attention network (Full-SA) whose each block connects to an individual SAM as in Fig.~\ref{fig:comparsion} (b). As illustrated in Fig.~\ref{fig:comparsion}, the implementation of SAMs incorporates three steps: extraction, processing, and recalibration. These operations and trainable components in SAMs require extra computational cost and parameters, resulting in slow inference and cumbersome network~\cite{bianco2018benchmark}. This limits self-attention usability on industrial applications that need a real-time response or small memory consumption, such as robotics, self-driving car, and mobile device. Therefore, other than only improving the capacity of the SAM, previous works also focus on the light-weight SAM design, \textit{e.g.}, reducing the parameters of an individual SAM~\cite{li2019spatial,lee2019srm}. 
\begin{figure}[t]
\centering
\includegraphics[width=0.94\linewidth]{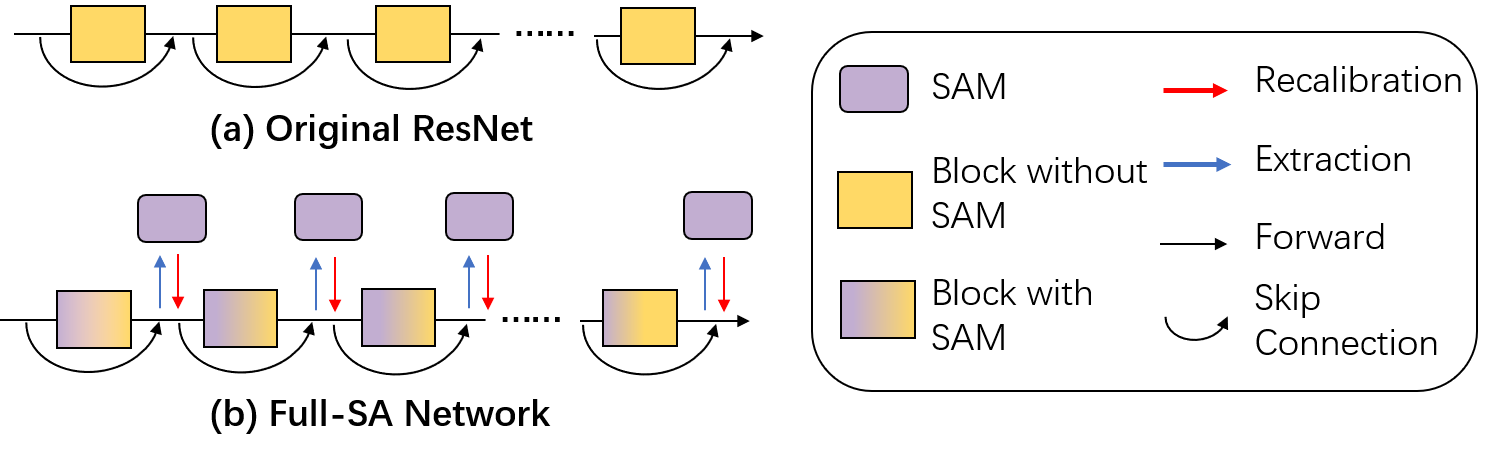}
\caption{ (a) Original ResNet; (b) Full-SA network. A network is called a Full-SA network if the SAM is individually defined for each block. The SAM can be divided into three steps~\cite{huang2020dianet}: (1)~Extraction: the plug-in module extracts internal features of a network by computing their statistics, like mean, variance; (2)~Processing: the SAM utilizes the extracted features to adaptively generate a mask via a trainable module; (3)~Recalibration: the mask is used to calibrate the feature maps by element-wise multiplication or addition. }
\vspace{-0.23cm}
\label{fig:comparsion}
\end{figure}


However, the extra cost of the lightweight SAM is still non-negligible for a deep network~\cite{bianco2018benchmark,li2019spatial}. The main reason lies in the conventional paradigm where the SAMs are individually plugged into every block of a CNN for granted~\cite{woo2018cbam,hu2018squeeze}. The additional inference time and the number of parameters increase with the growth of the network depth, which creates a bottleneck when applying SAMs to a deep network.
On the other hand, the network compression algorithms, such as network pruning~\cite{hecap,liu2018rethinking,molchanov2016pruning} and neural architecture search~\cite{Vidnerov2020Multi,guo2020single}, effectively reduce the network size by removing the redundant components
 while maintaining the accuracy of the slimmed network. Note that these techniques inherently alternate the connections between different ingredients (\textit{e.g.}, neurons, layers, or weights) in the CNN. This inspires us to pay more attention to the connections between the SAMs and the CNN backbone instead of the individual SAM design. If the number of connections is lessened, the computation and parameter cost will be reduced obviously. 
Based on these motivations, we articulate the Lottery Ticket Hypothesis for Self-attention Networks~(LTH4SA): 

\vspace{0.22cm}
\textit{A full self-attention (Full-SA) network contains a subnetwork with sparse self-attention connections that can (1) accelerate inference,~(2) reduce extra parameter increment, and~(3) maintain accuracy.}
\vspace{0.22cm}

Every SAM connects or disconnects to the block, and we call the set of these connection states for a CNN as a connection scheme (see Section~\ref{sec:connectionscheme}). A connection scheme is called a ticket if the self-attention subnetwork with this scheme satisfies the three above-mentioned conditions of LTH4SA. In Section~\ref{sec:lth4sa}, for the first time we both empirically and formally investigate the existence of LTH4SA. 
Our main observations are: (1) Empirically, there exist some self-attention subnetworks with sparse connection schemes achieving even better accuracy than the full self-attention network, which is also supported by our theoretical evidence where the large network can be approximated by its subnetwork; (2) The additional statistical analysis of the connection scheme shows that no specific block of the CNN dominates the accuracy when connecting SAM to the block. 
These observations indicate the key to obtaining a ticket is how to combine different connections of block and the SAMs. Certainly, finding the combination of the blocks and the SAMs to be a ticket is equivalent to solving a searching problem, and the corresponding algorithm should satisfy the following requirements based on the definition of LTH4SA and our empirical observations: (1) The searched connection scheme should be sparse and accurate enough to be a ticket; (2) The algorithm itself should have sufficient capacity to cover diverse connection schemes.
Therefore, we propose a simple yet effective reinforcement-learning-based~(RL-based) baseline method to search for a ticket given that the RL-based method can naturally handle multi-targets searching problems, and the rewards are designed exactly based on the requirements mentioned above. We call our proposed baseline method Efficient Attention Network~(EAN). In Section~\ref{sec:prelim}, we will briefly review the formulation of self-attention networks. Our proposed method for searching a ticket is introduced in Section~\ref{sec:method} and extensive experiments on widely-used benchmark datasets and popular self-attention networks are shown in Section~\ref{sec:result}. The property of EANs will be discussed in Section~\ref{sec:analysis}. Finally, we discuss the related works in Section~\ref{sec:related}. We summarize \textbf{our contribution} as follows: 

\begin{enumerate}

    \item We empirically find some counterintuitive phenomena: (a) Connecting the SAMs to all the blocks may not bring the largest performance boost; (b) Some connection schemes are harmful.
    
    \item  We propose a lottery ticket hypothesis for self-attention networks and provide both numerical and theoretical evidence for the existence of the ticket. Besides, we propose an effective searching method as a baseline to obtain a ticket and avoid harmful connection schemes. 
\end{enumerate}

\begin{figure*}[ht]
\centering
\includegraphics[width=0.9\linewidth]{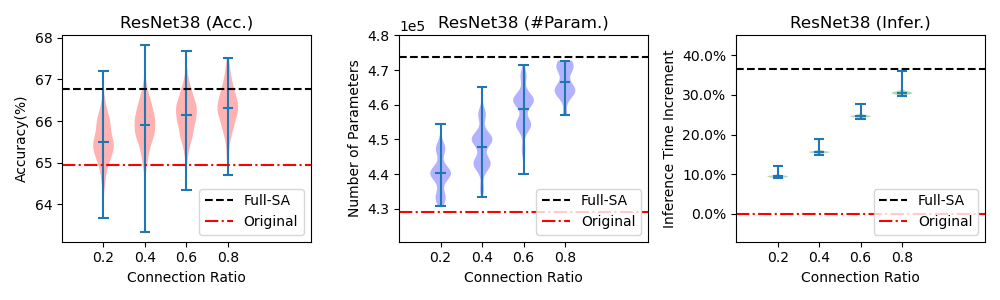}
\includegraphics[width=0.9\linewidth]{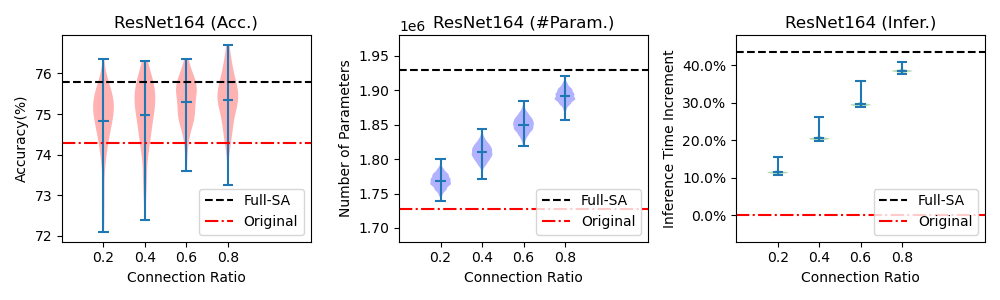}
\caption{Violin plot of accuracy, number of parameters, and inference time increment under different connection ratios. Three bars of each vertical line from the top to the bottom represent maximum, mean, and minimum, respectively. The light color shows the distribution. The black dotted line is the performance of the Full-SA network while the red line is the performance of the original CNN. The network with accuracy higher than the black dotted line is a ticket. }
\label{fig:sanxiantu}
\end{figure*}

\section{Preliminaries}         
\label{sec:prelim}
In this section, we first briefly review ResNet~\cite{he2016deep} and the Full-SA network and then introduce the connection scheme. 

\subsection{ResNet and the Full Self-attention (Full-SA) Network}
\noindent\textbf{ResNet.} The structure of a ResNet is shown in Fig.~\ref{fig:comparsion}~(a). In general, the ResNet has several stages, and each stage, whose feature maps have the same size, is a collection of consecutive blocks. Suppose a ResNet has $m$ blocks. Let $x_\ell$ be the input of the $\ell^{\rm th}$ block and $f_\ell(\cdot)$ be the residual mapping, then the output $x_{\ell+1}$ of the $\ell^{\rm th}$ block is defined as 
\begin{equation}
    x_{\ell+1} = x_\ell + f_\ell(x_\ell).
\end{equation}

\noindent\textbf{Full Self-attention (Full-SA) Network. } 
 A network is called a Full-SA network if the SAM is individually defined for each block as Fig.~\ref{fig:comparsion} (b). Note that the term ``full'' refers to a scenario when all blocks in a network connect to the SAMs.
Many popular SAMs adopt this way to connect with the ResNet backbone~\cite{hu2018squeeze,woo2018cbam}. We denote the SAM in the $\ell^{\rm th}$ block as $M(\cdot; W_\ell)$, where $W_\ell$ are the parameters. Then the attention will be formulated as $M(f_\ell(x_\ell);W_\ell)$ which consists of the extraction and processing operations introduced in Fig.~\ref{fig:comparsion}. In the recalibration step, the attention is applied to the residual output $f_\ell(x_\ell)$, \textit{i.e.},
\begin{equation}
    x_{\ell+1} = x_\ell + M(f_\ell(x_\ell);W_\ell)\odot f_\ell(x_\ell),
    \label{eqn:org-full}
\end{equation}
where $\ell=1,..., m$ and $\odot$ is the element-wise multiplication. Eq.(\ref{eqn:org-full}) indicates that the computational cost and the number of parameters grow with the increasing number of blocks $m$. 

\subsection{Connection Scheme}\label{sec:connectionscheme} Suppose that a ResNet has $m$ blocks. A sequence $\mathbf{a} = (a_1,a_2,\cdots,a_m)$ denotes a connection scheme, where $a_i =1$ if the $i^\text{th}$ block is connected to a SAM, otherwise it equals 0. A subnetwork specified by a scheme $\mathbf{a}$ can be formulated by:
\begin{align}
\begin{split}
    x_{\ell+1} = x_\ell + \Big( a_\ell\cdot M(f_\ell(x_\ell);W_{\ell})&+(1-a_\ell)\cdot \mathbf{1} \Big) \odot f_\ell(x_\ell),
\end{split}
\end{align}
where $\mathbf{1}$ denotes an all-one vector and $\ell$ is from $1$ to $m$. In particular, it becomes a Full-SA network if $\mathbf{a}$ is an all-one vector, or an original ResNet if $\mathbf{a}$ is a zero vector. 

\section{Lottery Ticket Hypothesis for Self-attention }
\label{sec:lth4sa}
In this section, we first study the existence of LTH4SA from empirical and theoretical perspectives. Then we investigate which block we should connect the SAMs to such that the corresponding connection scheme can achieve good accuracy. 

\subsection{Empirical evidence of LTH4SA Existence}
\label{sec:empistudy}
We empirically validate the proposed LTH4SA by investigating the accuracy of the self-attention subnetwork under different connection schemes. 


Specifically, we conduct classification on CIFAR100 using SAMs with the shallow and deep network backbones, \textit{i.e.,} ResNet38 and ResNet164, under different SAM connection ratios, \textit{i.e.,} the ratio of the number of connections to the number of blocks. 
Squeeze-and-Excitation SAM \cite{hu2018squeeze} is used in our experiments. We traverse all the connection schemes for ResNet38. However, as ResNet164 contains 54 blocks, there are $2^{54}$ different connection schemes. Traversal of all schemes is infeasible, and hence we randomly sample 100 connection schemes under each ratio. For simplicity and clear clarification, we choose the connection ratio 0.2, 0.4, 0.6, 0.8 to present the empirical results given that they are sufficient to cover different sparsity levels. Fig.~\ref{fig:sanxiantu} displays the distribution of accuracy, the number of parameters, and inference time increment under different connection ratios. From Fig.~\ref{fig:sanxiantu}, we observe that: 

\begin{enumerate}
    \item For different SAM connection ratios, there exist self-attention subnetworks with higher accuracy than the full self-attention network, even when the connections are very sparse, \textit{e.g.}, the connection ratio is 0.2. 
    \item For different SAM connection ratios, there exist self-attention subnetworks with lower accuracy than the original CNN, which means some connection schemes are harmful.
    \item Under the same connection ratio, the parameter and inference time of different connection schemes vary considerably. 
\end{enumerate}



Observation 1 shows the existence of the subnetwork with sparse connections yet good accuracy, which empirically illustrates the potential of finding some connection schemes that satisfy LTH4SA.
Moreover, Observation 2 and Observation 3 reveal that even though there exist some connection schemes that satisfy the three conditions of LTH4SA, it is still necessary to carefully design methods to find tickets as the sparse connection scheme may harm the accuracy of the network and incur larger parameters and computational cost. 


\subsection{Theoretical Evidence of LTH4SA Existence}
\label{sec:theorstudy}
In Section~\ref{sec:empistudy}, the empirical evidence of the existence of LTH4SA has been demonstrated. Now we provide some theoretical evidence of the existence as well. The hidden neuron can be considered as a SAM as they apply internal information from the previous layer to the following network outputs. 
Hence, we can consider a more general scenario that given a network, there exists a subnetwork that approximates the original network and can be obtained by removing the hidden neurons from the original network. 

We first consider a 1-hidden-layer feed-forward network and the network follows the initialization as \cite{du2018gradient}. 

\begin{theorem}
A 1-hidden-layer feed-forward NN is defined as $NN(x)=W^{2}\sigma(W^1x)$, where input $x\in R^{d}$ with $\|x\|_2\leq 1$, $W^1$ is of size $m\times d$, $W^2$ is of size $1\times m$ and $\sigma$ is ReLU activation. $W^1_{i,j}$ is initialized i.i.d. by the Gaussian distribution $\mathcal{N}(0, {(\frac{1}{\sqrt{m}})}^{2})$, and $W^2_{1,j}$ is initialized by the uniform distribution $Uniform\{1,-1\}$. Let $\mathcal{P}(d-1,\epsilon)$ be $\mathbb{P}\{\chi^2(d-1)\geq \epsilon^2\}$, where $\chi^2(d-1)$ is a chi-square variable with $d-1$ degree of freedom. Then for any $\epsilon, \delta>0$, when the number of hidden neurons $m>\frac{\ln(\delta)}{\ln(\mathcal{P}(d-1,\epsilon))}$, then there exists the row $j$ of $W^1$ such that when we set the row $j$ to be zero, i.e., $B_jW^1$ with $B_j=diag\{1,\cdots,1,0,1,\cdots,1\}$ (the $j ^{\rm th}$ entry is 0), we have 
\begin{align*}
    \|W^{2}\sigma(W^1x)-W^{2}\sigma(B_jW^1x)\|_2<\epsilon,
\end{align*}
with probability higher than $1-\delta$. 
\label{theo:1}
\end{theorem}

The proof for Thm.~\ref{theo:1} is in Appendix. Thm.~\ref{theo:1} shows that when the width of the network is sufficiently large, we can find a subnetwork that approximates the original network with high probability. Next, we consider a more general and modern network structure, \textit{i.e.}, the ResNet~\cite{he2016deep} with ReLU. 

\begin{theorem}

Let $T(x)$ be a Lipschitz continuous and Lebesgue integrable function in $d$-dimensional compact set $K$. And $R_{\text{full}}(x,\theta_{\text{full}})$ is a ReLU ResNet structure with parameters $\theta_{\text{full}}$. Let $\epsilon_0>0$ be a constant. Suppose that there exists $\theta_{\text{full}}^0$ such that
$
    \int_{K}|R_{\text{full}}(x,\theta_{\text{full}}^0) - T|dx \leq \frac{\epsilon_0}{2}.
$
If the width of each layer in $R_{\text{full}}(x,\theta_{\text{full}})$ is larger than $d$ and the depth of $R_{\text{full}}(x,\theta_{\text{full}})$ is larger than a constant that depends on $\epsilon_0$, then for any $\epsilon \in (\epsilon_0, 1)$, there exists a subnetwork $R_\text{sub}(x)$ of $R_{\text{full}}(x,\theta_{\text{full}})$ such that 
\begin{equation}
    \int_{K}|R_{\text{full}}(x,\theta_{\text{full}}^0) - R_\text{sub}(x)|dx \leq \epsilon.
\end{equation}
\label{theo:2}
\end{theorem}
The proof for Thm.~\ref{theo:2} is in Appendix. In industrial applications, it is not necessary for the discrepancy between the output of two networks to be arbitrarily small if they have comparable performance. In practice, the discrepancy is acceptable if it reaches some certain levels, such as $\epsilon_0=10^{-5}$ or $10^{-10}$. When a sufficiently small discrepancy $\epsilon_0$ is given, Thm.~\ref{theo:2} can guarantee that a large-size network contains a subnetwork that has similar performance.
\subsection{Which block should we connect the SAMs to?}
\label{sec:which_block}


By studying the statistical characteristics of tickets and some harmful connection schemes, in this part, we investigate which block we should connect the SAMs to such that the corresponding connection scheme can achieve good accuracy. 

We consider a statistics called connection score which characterizes the frequency of the connections of a scheme set. Given a network with $m$ blocks and  a set of $N$ connection schemes with each scheme $\mathbf{a}_i = (a_{i1},a_{i2},...,a_{im}), a_{ij} \in\{0,1\},i=1,...,N, j=1,...,m$, we define the connection score of a scheme set as follows,
\begin{equation}
    \left(\frac{1}{N}\sum_{i=1}^Na_{i1},\frac{1}{N}\sum_{i=1}^Na_{i2},...,\frac{1}{N}\sum_{i=1}^Na_{im}\right).
\end{equation}
This statistic can characterize the importance of each block based on their frequency of connecting with the SAM. If the connection score of a block is large, the connection between this block and the SAM appears in large portion among $N$ connection schemes. 

We consider two sets of connection schemes, \textit{i.e.}, the ticket set and the bad scheme set, for ResNet38 as presented in Fig.~\ref{fig:full_ori}. The ticket set stands for the set of schemes that satisfy the LTH4SA, while the bad scheme set stands for a set of schemes whose accuracy is lower than the original network. 

From Fig.~\ref{fig:full_ori}, we can observe that the connection score of each block is almost the same for both the ticket set and bad scheme set. Besides, for each set, we use univariate linear regression to fit these scores, and use slope to characterize their trends. We can see that the slopes of two scheme set are close to zero. These observations indicate no specific block of the network will dominate the accuracy, and each block can be connected to the SAMs with almost equal frequency in a ticket. 


 \begin{figure}[t]
    \centering
    \includegraphics[width=0.95\linewidth]{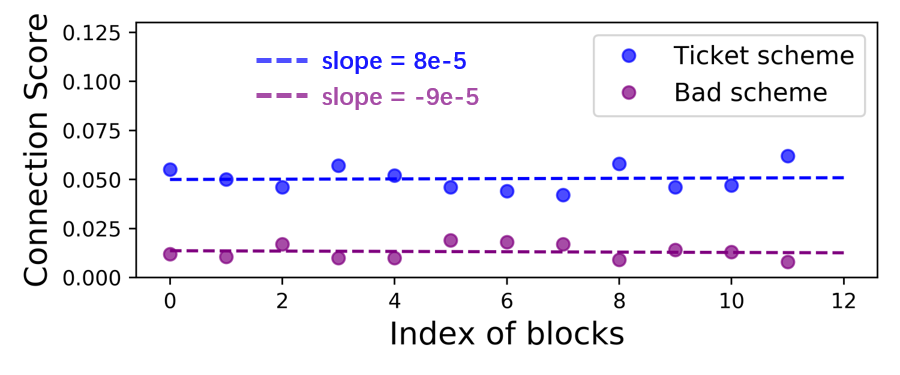}
    \caption{The comparison of the connection scores of the ticket set or the bad scheme set with ResNet38.}
    \label{fig:full_ori}
\end{figure}
Since no specific block dominates the accuracy, it is not easy to define a metric to identify the importance for the connection of each block as the network pruning algorithms do~\cite{huang2021rethinking,he2019filter}. Hence, it is necessary to design an effective search method to find the tickets from the thousands of possible connection schemes.

\section{Proposed baseline method}
\label{sec:method}


In this section, we introduce the proposed method which consists of two parts. First, we pre-train a supernet as the search space. The supernet assembles different candidate network architectures into a single network by weight sharing~\cite{pmlr-v139-wang21i}. Each candidate architecture corresponds to a subnetwork and in our problem, each connection scheme corresponds to a specific subnetwork sampled from the supernet. Second, we use a policy-gradient-based method to search for an optimal connection scheme from the supernet. The basic workflow of our method is shown in Alg.\ref{alg:ean}. 

\begin{algorithm}[t]  
    \caption{Searching a ticket of LTH4SA}
    \label{alg:ean}   
    \textbf{Input:} Training set $D_\text{train}$; validation set $D_\text{val}$; a Full-SA attention network $\Omega(\mathbf{x}|\mathbf{1})$; the pre-training step $K$; the searching step $T$; the probability of retaining connection $\beta$. 
    
    \textbf{Output:} The trained controller $\chi_\theta(q_0)$.
        
    \begin{algorithmic}[1]

    \State\algorithmiccomment{Pre-train the supernet}

    \For{$t$ from 1 to $K$} 

        \State $\mathbf{a} \sim [Bernoulli(\beta)]^m$
        \State train $\Omega(\mathbf{x}|\mathbf{a})$ with $D_\text{train}$
    \EndFor  
    
    \State\algorithmiccomment{Policy-gradient-based search}       
       \For{$t$ from 1 to $T$} 

        \State $\mathbf{p}_\theta \gets \chi_\theta(q_0)$
        \State $\mathbf{a} \sim \mathbf{p}_\theta$
        \State Calculate the rewards $g_\text{spa}$,$g_\text{val}$,$g_\text{rnd}$ (see Appendix)
        \State Update the trainable parameters.
       

    
    \EndFor
    \State\Return $\chi_\theta(q_0)$
    \end{algorithmic} 
\end{algorithm}

\subsection{Problem Description}
\label{sec:setup}
According to LTH4SA, our goal is (1) to find a connection scheme $\mathbf{a}$, which is sparse enough for less computational cost and parameters, from $2^m$ possibilities; (2) to ensure that the subnetwork specified by the scheme can maintain the accuracy as the Full-SA network. 

To determine the optimal architecture from the pool of candidates, it is costly to evaluate all the candidates' performances after training from scratch, since even training a candidate individually from scratch will require a large number of computation times~(e.g., tens of hours), not to mention traversing such an extensive pool. In many related works on Neural Architecture Search~(NAS), the validation accuracy of the candidates sampled from a supernet can be served as a satisfactory performance proxy~\cite{guo2020single,you2020greedynas,chu2019fairnas} to approximately estimate those candidates' stand-alone\footnote{Train the subnetworks from scratch} performance, which can effectively reduce the extensive computational cost correspondingly. Thus similarly, to efficiently obtain the optimal connection scheme, we propose to train the supernet as the search space. We follow the DropAct~\cite{liang2021drop} training strategy to train the supernet. 
 Then we consider the validation performance of the sampled subnetworks from the supernet as the proxy for their stand-alone performance.
We consider a supernet $\Omega(\mathbf{x}|\mathbf{a})$ with $m$ blocks and input $\mathbf{x}$. $\Omega(\mathbf{x}|\mathbf{a})$ has the same components as a Full-SA network, but its connections between blocks and SAMs are specified by $\mathbf{a}$.  



\subsection{Pre-training the Supernet}
\label{sec:pretrain}

Given a dataset, we split all training samples into the training set $D_\text{train}$ and the validation set $D_\text{val}$. To train the supernet, we activate or deactivate the SAM in each block of it randomly during optimization. Specifically, we first initialize a supernet $\Omega(\mathbf{x}|\mathbf{a}^{(0)})$, where $\mathbf{a}^{(0)} = (1,\cdots,1)$. At the iteration $t$, we randomly draw a connection scheme $\mathbf{a}^{(t)}=(a^t_1, \cdots, a^t_m)$, where $a^t_i$ is sampled from a Bernoulli distribution $Bernoulli(\beta)$. 
Then we train subnetwork $\Omega(\mathbf{x}|\mathbf{a}^{(t)})$ with the scheme $\mathbf{a}^{(t)}$ from the supernet on $D_\text{train}$ via weight sharing. More detail of training the supernet is provided in Appendix.

\subsection{Training Controller with Policy Gradient}
We introduce the steps to search for the optimal connection scheme. Concretely, we use a controller to generate connection schemes and update the controller by policy gradient. 

We use a fully connected network as the controller $\chi_\theta(q_0)$ to produce the connection schemes, where $\theta$ are the learnable parameters, and $q_0$ is a constant vector $\mathbf{0}$. 
The output of $\chi_\theta(q_0)$ is $\mathbf{p_\theta}$, where $\mathbf{p_\theta} = (p_\theta^1,p_\theta^2,...,p_\theta^m)$ and $p_\theta^i$ represents the probability of connecting the SAM to the $i^\text{th}$ block. A realization of $\mathbf{a}$ is sampled from the controller output, \textit{i.e.}, $\mathbf{a} \sim \mathbf{p}_\theta$. The probability associated with the scheme $\mathbf{a}$ is $\mathbf{\hat{p}_\theta} = (\hat{p}_\theta^1,\hat{p}_\theta^2,...,\hat{p}_\theta^m)$, where $\hat{p}_\theta^i = (1-a_i)(1-p_\theta^i)+a_ip_\theta^i$. 

We denote $G(\mathbf{a})$ as a reward for $\mathbf{a}$. The parameter set $\theta$ within the controller can be updated via policy gradient with learning rate $\eta$, \textit{i.e.}, 
\begin{align}
\begin{split}
    R_\theta = G(\mathbf{a})\cdot\sum_{i=1}^m\log \hat{p}_\theta^i,\quad
     \theta \gets \theta + \eta\cdot \nabla {R}_\theta.
\end{split}
    \label{eqn:policy_gradient}
\end{align}
In this way, the controller tends to output the probability that results in a large reward $G$. Therefore, designing a reasonable $G$ can help us search for a good structure. 

To find a ticket, we should incorporate the accuracy and connection ratio into the reward $G$. We use the validation accuracy $g_\text{val}$ of the subnetwork $\Omega(\mathbf{x}|\mathbf{a})$ sampled from the supernet as a reward, which depicts the performance of its structure. Besides, we complement a sparsity reward $g_\text{spa}$ to encourage the controller to generate the schemes with fewer connections between SAMs and backbone. Finally, to encourage the controller to explore more potentially useful connection schemes, we add the Random Network Distillation~(RND) curiosity bonus $g_\text{rnd}$ in our reward~\cite{burda2018exploration}. Therefore, $G(\mathbf{a}) = \lambda_1\cdot g_\text{spa} + \lambda_2 \cdot
    g_\text{val}+ \lambda_3 \cdot g_\text{rnd}$, 
where $\lambda_1, \lambda_2, \lambda_3$ are the coefficient for each bonus. The detailed definition of $g_\text{spa},g_\text{val},$ and $g_\text{rnd}$ can be found in Appendix. 
  
\section{Experiments}
\label{sec:result}
In this section, we demonstrate the effectiveness of our method in finding the ticket. First, we show that our method can outperform some popular NAS and pruning algorithms. Next, to further reduce the number of parameters for various types of SAMs, we search the ticket from another self-attention framework proposed in \cite{huang2020dianet}. Finally, we conduct a comprehensive comparison with various searching methods.

\subsection{Datasets and Settings} 

 On CIFAR100 \cite{cifar} and ImageNet2012~\cite{ILSVRC15} datasets, we conduct classification using ResNet~\cite{he2016deep} backbone with different SAMs, including Squeeze-and-Excitation (SE) \cite{hu2018squeeze}, Spatial Group-wise Enhance (SGE) \cite{li2019spatial} and Dense-Implicit-Attention (DIA) \cite{huang2020dianet} modules. The description of these SAMs is in Appendix. 
 Since the networks with SAMs have extra computational cost compared with the original backbone inevitably, we formulate the relative inference time increment to represent the relative speed of different self-attention networks, \textit{i.e.},
 \begin{equation}
 \frac{I_t({\rm CNN \ with \ SAMs}) - I_t({\rm Original \ CNN})}{I_t({\rm Original \ CNN})} \times 100\%,
 \label{eqn:inference}
 \end{equation}
where $I_t(\cdot)$ denotes the inference time of the network.
The inference time is measured by forwarding the data of batch size 50 for 1000 times. 

 \textbf{CIFAR100.} CIFAR100 consists of 50k training images and 10k test images of size 32 by 32. In our implementation, we choose 10k images from the training images as a validation set (100 images for each class, 100 classes in total), and the remainder images as a sub-training set. Regarding the experimental settings of ResNet164~\cite{he2016deep} backbone with different SAMs, the supernet is trained for 150 epochs, and the search step $T$ is set to be 1000. 
 
 \textbf{ImageNet2012.} ImageNet2012 comprises 1.28 million training images. We split 100k images (100 from each class and 1000 classes in total) as the validation set and the remainder as the sub-training set. The testing set includes 50k images. Besides, the random cropping of 224 by 224 is used. Regarding the experimental settings of ResNet50~\cite{he2016deep} backbone with different SAMs, the supernet is trained for 40 epochs, and the search step $T$ is set to be 300. 
 

\begin{table*}
\caption{Comparison of relative inference time increment (denoted by Infere. (\%) as in Eq.(\ref{eqn:inference})), the number of parameters (\#P (M)), and test accuracy (Acc.) on CIFAR100. Here the connection schemes are searched with SE module on ResNet38 and ResNet164, using different searching methods. ``Ticket?'' presents whether the found connection scheme is a ticket or belongs to top 5\% high accuracy. We train a supernet with the probability  $\beta$ of retaining connection as in Alg.~\ref{alg:ean}.}
  \centering
  \begin{adjustbox}{width=0.95\textwidth,center}
    \begin{tabular}{clcccllrclcccl}
\cmidrule{1-7}\cmidrule{9-14} 
\cmidrule{1-7}\cmidrule{9-14}  \multicolumn{7}{c}{ResNet38}                          &       & \multicolumn{6}{c}{ResNet164} \\
\cmidrule{1-7}\cmidrule{9-14}    $\beta$ & \multicolumn{1}{c}{Method} & Acc.  & \#P (M)   & Infere. (\%) & \multicolumn{1}{c}{Top5\%?} & \multicolumn{1}{c}{Ticket?} &       & $\beta$ & \multicolumn{1}{c}{Method} & Acc.  & \#P (M)  & Infere. (\%)  & \multicolumn{1}{c}{Ticket?} \\
\cmidrule{1-7}\cmidrule{9-14}    \multicolumn{1}{c}{\multirow{7}[4]{*}{\newline{}0.2}} & ENAS  & 64.94 & 0.43      & 0.00      & \XSolidBrush & \XSolidBrush &       & \multicolumn{1}{c}{\multirow{7}[4]{*}{\newline{}0.2}} & ENAS  & 74.29 & 1.73      & 0.00       & \XSolidBrush \\
          & DARTS & 66.07 & 0.44     & 13.04      & \XSolidBrush & \XSolidBrush &       &       & DARTS & 74.42 & 1.74      & 8.43       & \XSolidBrush \\
          & GA    & 65.25 & 0.47 &	32.88      & \XSolidBrush & \XSolidBrush &       &       & GA    & 76.07 &	1.80 &	15.64       & \Checkmark \\
          & $\ell_1$ & 66.06 &	0.45 &	6.37      & \XSolidBrush & \XSolidBrush &       &       & $\ell_1$ & 74.50	&1.81 	&8.92       & \XSolidBrush \\
           & GM    & 66.39	& 0.45 &	6.45      & \XSolidBrush & \XSolidBrush &       &       & GM    & 75.09 &	1.81 &	8.24      & \XSolidBrush \\
\cmidrule{2-7}\cmidrule{10-14}          & EAN   & \textbf{66.78} & 0.45 &	23.21      & \XSolidBrush & \Checkmark &       &       & EAN   & \textbf{76.53} & 1.85 &	25.84       & \Checkmark \\
\cmidrule{1-7}\cmidrule{9-14}    \multicolumn{1}{c}{\multirow{7}[4]{*}{\newline{}0.5}} & ENAS  & 65.09 &	0.45 &	26.55      & \XSolidBrush & \XSolidBrush &       & \multicolumn{1}{c}{\multirow{7}[4]{*}{\newline{}0.5}} & ENAS  & 75.33&	1.81& 	18.13       & \XSolidBrush \\
          & DARTS & 66.06 &	0.44 & 20.02      & \XSolidBrush & \XSolidBrush &       &       & DARTS & 73.44 &	1.85 &	21.97       & \XSolidBrush \\
          & GA    & 65.57 &	0.44 &	19.82      & \XSolidBrush & \XSolidBrush &       &       & GA    & 75.75 &	1.76 &	15.21       & \XSolidBrush \\
          & $\ell_1$ & 65.86 &	0.47 &	20.00      & \XSolidBrush & \XSolidBrush &       &       & $\ell_1$ &73.79	 &1.90 &	23.23       & \XSolidBrush \\
           & GM    & 66.27	& 0.46 &	19.50      & \XSolidBrush & \XSolidBrush &       &       & GM    & 75.37 &	1.90 &	22.04       & \XSolidBrush \\
\cmidrule{2-7}\cmidrule{10-14}          & EAN   & \textbf{66.90} & 0.45 &	26.52  & \Checkmark & \Checkmark &       &       & EAN   & \textbf{76.21} & 1.82 &	22.41       & \Checkmark \\
\cmidrule{1-7}\cmidrule{9-14}    \multicolumn{1}{c}{\multirow{7}[4]{*}{\newline{}0.8}} & ENAS  & 66.77 & 0.47 &	38.98      & \XSolidBrush & \XSolidBrush &       & \multicolumn{1}{c}{\multirow{7}[4]{*}{\newline{}0.8}} & ENAS  & 75.80 & 1.93 &	43.56       & \XSolidBrush \\
          & DARTS & 66.67&	0.46& 	29.55      & \XSolidBrush & \XSolidBrush &       &       & DARTS & 75.42&	1.90 &	34.42       & \XSolidBrush \\
          & GA    & 66.21 & 0.46 &	33.61      & \XSolidBrush & \XSolidBrush &       &       & GA    & 75.10 & 1.78 &	13.91       & \XSolidBrush \\
          & $\ell_1$ & 66.32 & 0.47 &	29.60      & \XSolidBrush & \XSolidBrush &       &       & $\ell_1$ & \textbf{76.27} &	1.92 &	34.66       & \Checkmark \\
        &  GM    & 66.61 &	0.47 &	29.60      & \XSolidBrush & \XSolidBrush &       &       & GM    & 75.55 & 1.92 & 	34.29       & \XSolidBrush \\
\cmidrule{2-7}\cmidrule{10-14}          & EAN   & \textbf{67.06} & 0.46 &	19.98      & \Checkmark & \Checkmark &       &       & EAN   & 75.80 & 1.82 &	20.64       & \Checkmark \\
\cmidrule{1-7}\cmidrule{9-14} 
        & Original CNN   & 64.94 & 0.43 &	0.00      & - & - &       &       & Original CNN   & 74.29 & 1.73 &	0.00       & -\\
\cmidrule{1-7}\cmidrule{9-14}  
\cmidrule{1-7}\cmidrule{9-14}  \end{tabular}%
\end{adjustbox}

  \label{tab:dabiao}%
\end{table*}%

 \subsection{Searching Tickets from the Full-SA Network}
 \label{sec:fullsa}
In this part, we compare our search method with some popular NAS and 
pruning algorithms, \textit{e.g.}, Genetic Algorithm (GA)~\cite{Vidnerov2020Multi}, ENAS~\cite{pham2018efficient} and DARTS~\cite{liu2018darts}, $\ell_1$ pruning, and Geometric Median (GM) pruning~\cite{he2019filter}. Table~\ref{tab:dabiao} displays the experiments conducted on CIFAR100 with Full-SA networks including ResNet38 and ResNet164 using SE for different searching methods under different supernets. 

From Table~\ref{tab:dabiao}, our method outperforms the heuristic method GA. Different from DARTS that searches schemes by minimizing the validation loss, the RL-based method~(e.g., ENAS and our baseline method) can directly consider the validation accuracy as a reward although the accuracy or sparsity constraint is not differentiable. However, ENAS does not learn an effective scheme from the reward because of its controller architecture as discussed in Section~\ref{sec:comp}. 

Since the pruning algorithms ($\ell_1$ and GM) are based on intuitive design \cite{huang2020convolution} to measure the importance of the connection individually and ignore the combination of the connections of a network as mentioned in Section~\ref{sec:which_block}, they may fail to find the reasonable connection scheme and obtain unstable results in some cases.

 \subsection{Searching Tickets from Full-Share Network}
 We demonstrate that our search method can also be applied to other self-attention frameworks, such as sharing mechanism~\cite{huang2020dianet}, which shares a SAM with the same set of parameters to different blocks in the same stage. We call a network as Full-Share network if each block in the same stage of this network connects to a shared SAM. The shared SAM significantly reduces the trainable parameters compared with the Full-SA network. 
 
   The connection scheme found by our method from a Full-Share network is called Share-EAN, as illustrated in Fig.~\ref{fig:comparsion2} (b). 
We show the test accuracy, the number of parameters, and relative inference time increment on CIFAR100 and ImageNet2012 of Share-EAN in Table~\ref{tab:all_results}, searching from the supernet trained with retaining ratio $\beta=0.5$. 

From Table~\ref{tab:all_results}, we can observe that most results of Share-EAN satisfy the three conditions of LTH4SA. Since both Share-EAN and the Full-Share network use sharing mechanism~\cite{huang2020dianet} to implement the SAM over the original ResNet, they both have fewer parameters increment than the Full-SA network. Besides, Share-EANs achieve faster inference speed among the Full-Share network and the Full-SA network. 
Furthermore, the accuracy of Share-EANs is on par or even surpass that of the Full-Share networks. 

\begin{figure}[t]
\centering
\includegraphics[width=1\linewidth]{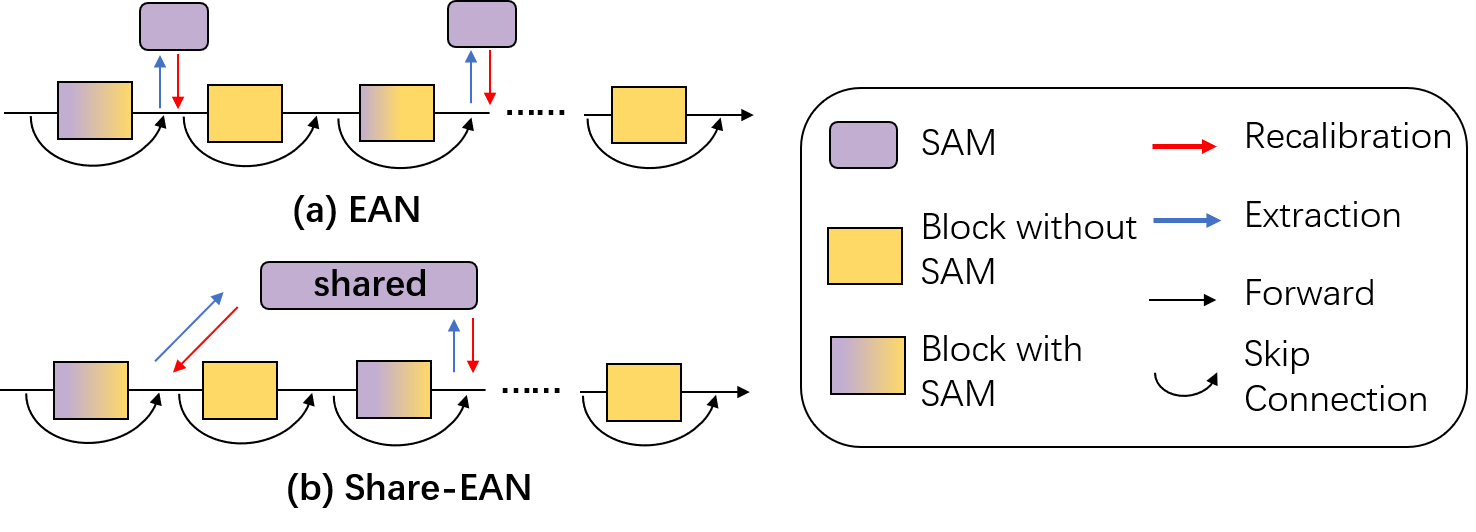}
\caption{The network structures of  (a) EAN and (b) Share-EAN.  }
\label{fig:comparsion2}
\end{figure}

 \begin{table*}[ht]
    \caption{Comparison of the relative inference time increment (see Eq.(\ref{eqn:inference})), the number of parameters, and test accuracy between various SAMs on CIFAR100 and ImageNet2012. ``Original CNN'' stands for ResNet164 backbone in CIFAR100 and ResNet50 backbone in ImageNet. The SAM in DIA is a RNN so DIA network does not have Full-SA structure.}
  \centering
  \begin{adjustbox}{width=\textwidth,center}
    \begin{tabular}{cllccclccclccc}
    \toprule
    \multirow{2}[4]{*}{Dataset} & \multicolumn{1}{c}{\multirow{2}[4]{*}{Model}} &       & \multicolumn{3}{c}{Test Accuracy (\%)} &       & \multicolumn{3}{c}{Parameters (M)} &       & \multicolumn{3}{c}{Relative Inference Time Increment (\%)} \\
\cmidrule{4-6}\cmidrule{8-10}\cmidrule{12-14}          &       &       & Full-SA & Full-Share & Share-EAN   &       & Full-SA & Full-Share & Share-EAN   &       & Full-SA & Full-Share & Share-EAN  \\
    \midrule
    \multirow{4}{*}{\rotatebox{90}{CIFAR100}} & Original CNN   &       & 74.29  & -     & -     &       & 1.727  & -     & -     &       & 0.00  & -     & - \\
          & SE~\cite{hu2018squeeze} &       & 75.80  & 76.09  & \textbf{76.93} &       & 1.929  & 1.739  & \textbf{1.739} &       & {43.56}  & 41.66  & \textbf{18.81}~({\color{black}{$\downarrow$ 22.85}}) \\
          & SGE~\cite{li2019spatial} &       & 75.75  & 76.17  & \textbf{76.36} &       & 1.728  & 1.727  & \textbf{1.727} &       & 93.60  & 93.41  & \textbf{50.49}~({\color{black}{$\downarrow$ 42.92}}) \\
          & DIA~\cite{huang2020dianet} &       & -     & \textbf{77.26} & 77.12  &       & -     & 1.946  & \textbf{1.946} &       & -     & 121.11  & \textbf{65.46}~({\color{black}{$\downarrow$ 55.65}}) \\
    \midrule
    \multirow{4}{*}{\rotatebox{90}{ImageNet}} & Original CNN   &       & 76.01  & -     & -     &       & 25.584  & -     & -     &       & 0.00  & -     & - \\
          & SE~\cite{hu2018squeeze}    &       & 77.01  & 77.35  & \textbf{77.40} &       & 28.115  & 26.284  & \textbf{26.284} &       & 25.94  & 25.92  & \textbf{10.35}~({\color{black}{$\downarrow$ 15.57}}) \\
          & SGE~\cite{li2019spatial} &       & 77.20  & 77.51  & \textbf{77.62} &       & 25.586  & 25.584  & \textbf{25.584} &       & 40.60  & 40.50  & \textbf{19.66}~({\color{black}{$\downarrow$ 20.84}}) \\
          & DIA~\cite{huang2020dianet}   &       & -     & 77.24  & \textbf{77.56} &       & -     & 28.385  & \textbf{28.385} &       & -     & 27.26  & \textbf{16.58}~({\color{black}{$\downarrow$ 10.68}}) \\
    \bottomrule
    \end{tabular}%
    \end{adjustbox}
  \label{tab:all_results}%
\end{table*}%

\begin{table*}[htbp]
  \caption{Comparison of the accuracy and relative inference time increment of the searched network for different methods.}
\vspace{-0.2cm}
  \centering
  \begin{adjustbox}{width=0.7\textwidth,center}
    \begin{tabular}{cccccc}
    \toprule
    Method & Acc.  & Time Increment (\%) & Method & Acc.  & Time Increment (\%) \\
    \midrule
    Share-EAN & \textbf{76.93}  & 18.81~({\color{black}{$\downarrow$22.85}}) & HSP~(101010..) & 75.02 &  21.29~({\color{black}{$\downarrow$20.37}}) \\
    DARTS & 75.41 & 28.02~({\color{black}{$\downarrow$13.64}}) & HSP~(010101..) & 74.87 & 21.29~({\color{black}{$\downarrow$20.37}}) \\
    GA    & 76.09 & 20.87~({\color{black}{$\downarrow$20.79}}) & HSP~(100100..) & 75.29 & 14.50~({\color{black}{$\downarrow$27.16}}) \\
     ENAS    & 76.08 & 28.05~({\color{black}{$\downarrow$13.61}}) & HSP~(010010..) & 74.01 & 14.50~({\color{black}{$\downarrow$27.16}})\\
    
    \bottomrule
    \end{tabular}%
  \end{adjustbox}

  \label{tab:searching}%
\end{table*}%

\subsection{Comprehensive Comparison with Searching Methods}
\label{sec:comp}
In this part, we compare our method (EAN) with heuristic selection policy (HSP), Genetic Algorithm (GA)~\cite{Vidnerov2020Multi}, ENAS~\cite{pham2018efficient} and DARTS~\cite{liu2018darts} for the Full-Share network. HSP is a heuristic policy that makes SAM connection every $N$ layers. For example, when $N=2$, the schemes can be $10101\cdots$ or $01010\cdots$.  Table~\ref{tab:searching} displays the experiments conducted on CIFAR100 with ResNet164 and SE module for different searching methods. From Table~\ref{tab:searching}, our method achieve better results of the Full-Share network compared with other methods, which is consistent with the results of the Full-SA network in Section~\ref{sec:fullsa}. Besides, the heuristic design of the connection scheme like HSP does not give a scheme for good accuracy, and hence it is necessary for a careful search algorithm as mentioned in \ref{sec:which_block}. 

The controller of ENAS tends to converge to some periodic-alike schemes at a fast speed. In this case, it will conduct much less exploration of the potential efficient structures. 
We show the list of connection schemes by ENAS~(an example) in Table~\ref{tab:rnn}. 
The majority of the schemes searched by ENAS are ``111...111''~(Full-Share network) or ``000...000''~(Original network), which shows that it can not strike the balance between the performance and inference time.  In Table~\ref{tab:ean_vs_enas}, the minority of the periodic-alike schemes searched by ENAS are shown, e.g., ``001'' in ENAS~(a). Such schemes may result from the input mode of ENAS, 
\textit{i.e.}, for a connection scheme $\mathbf{a} = (a_1,a_2,...,a_m)$, the value of component $a_l$ depends on $a_{l-1},a_{l-2},...,a_1$. This strong sequential correlations let the sequential information dominate in the RNN controller instead of the policy rewards. Compared with the periodic-alike connection schemes searched by ENAS, Share-EAN demonstrates better performance.

Besides, our experiment indicates that ENAS explores a much smaller number of candidate schemes. We quantify the convergence of the controller using $\bar{\mathbf{p}}=\frac{1}{m}\sum_{i=1}^m\hat{p}_\theta^i$, which is the mean of the probability $\hat{\mathbf{p}}$ associated with the scheme. When $\bar{\mathbf{p}}$ is close to 1, the controller tends to generate a deterministic scheme. Fig.~\ref{fig:enas_vs_ean} and Table \ref{tab:rnn} show the curve of $\bar{\mathbf{p}}$ with the growth of searching iterations, where $\bar{\mathbf{p}}$ of ENAS shows the significant tendency for convergence in 20 iterations and converges very fast within 100 iterations. Generally speaking, methods in NAS~\cite{pham2018efficient,zoph2016neural} require hundreds or thousands of iterations for convergence. 

 \begin{table*}[htbp]
  \caption{The connection schemes searched by ENAS~\cite{pham2018efficient} or our method. The experiment is conducted on CIFAR100 with SE module and ResNet164 backbone.}
  \centering
  \begin{adjustbox}{width=0.8\textwidth,center}
    \begin{tabular}{ccccc}
    \toprule
    Method & Stage1 & Stage2 & Stage3 & Test Accuracy (\%) \\
    \midrule
    ENAS (a) & 001001001001001001 & 001001001001001001 & 001001001001001001 & 75.80 \\
    ENAS (b) & 100100101100100100 & 101101100100101101 & 100100101101100100 & 75.11 \\
    ENAS (c) & 110110110110110110 & 110110110110110110 & 110110110110110110 & 76.08 \\
    \midrule
    Our method (a) & 001100100101110101 & 001100000111001111 & 101100000111110001 & \textbf{76.93} \\
    Our method (b) & 001100000001010111 & 011100001000010111 & 101000100110000000 & \textbf{76.71} \\
    \bottomrule
    \end{tabular}%
    \end{adjustbox}

  \label{tab:ean_vs_enas}%
\end{table*}%

 \begin{figure}
    \centering
    \includegraphics[width=0.9\linewidth]{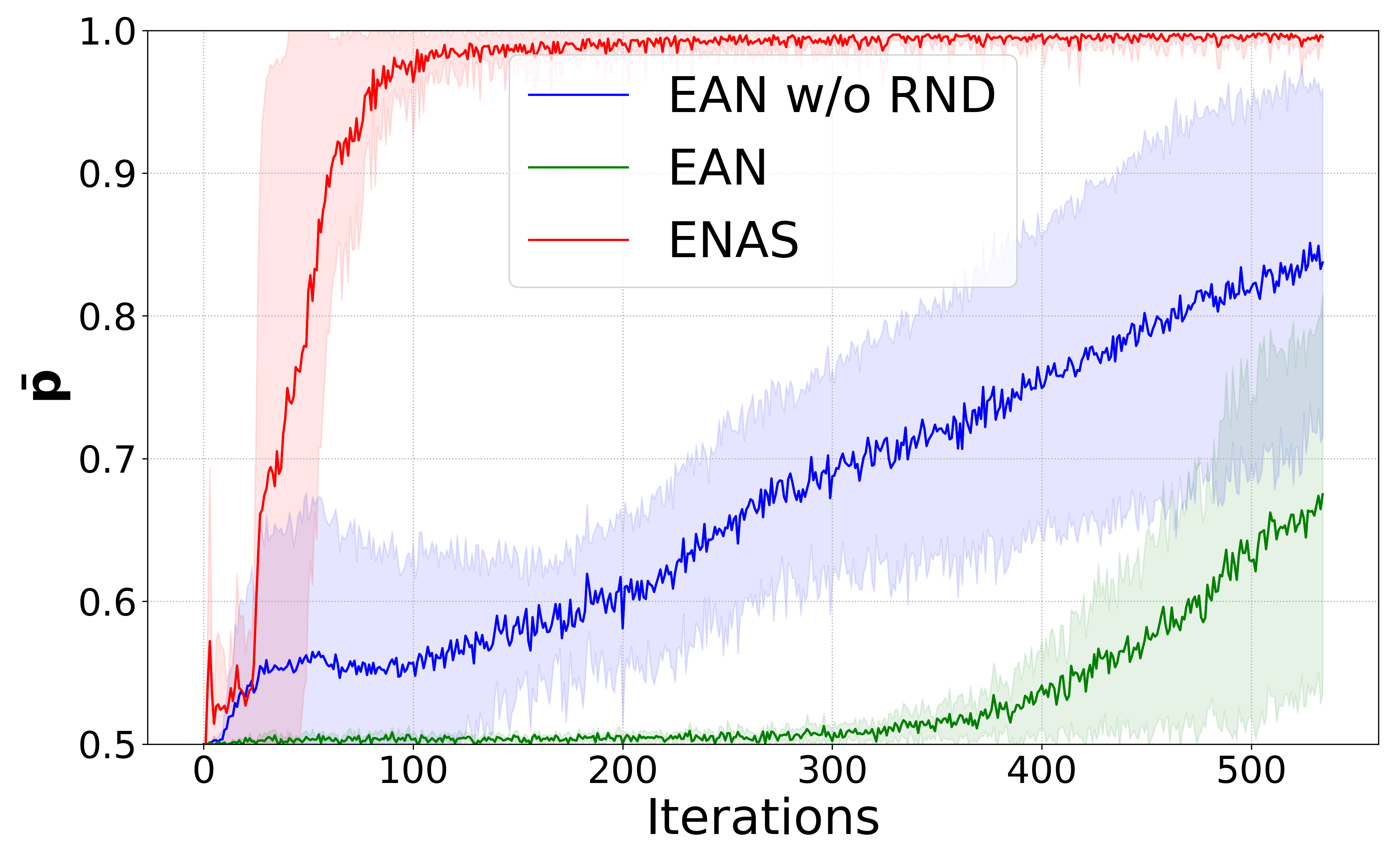}
    \caption{Comparison of the convergence speed between ENAS and EAN. The controller tends to generate a deterministic scheme when $\bar{\mathbf{p}}$ is close to 1.}
    \label{fig:enas_vs_ean}
\end{figure}
    
\begin{table}
  \caption{The connection scheme searched by ENAS. $\mathbf{\bar{p}}$ is the average of the probability associated with the scheme. The controller tends to generate a deterministic scheme if $\bar{\mathbf{p}}$ is close to 1. The experiment is conducted on CIFAR100 with ResNet164 and SE modules.}
  \centering
  \begin{adjustbox}{width=0.5\textwidth,center}
    \begin{tabular}{cccc}
    \toprule
    \textbf{Iteration} & \textbf{Connection Scheme} & \textbf{Sparse} & $\mathbf{\bar{p}}$ \\
    \midrule
    0     & 000110000001111101110010000001000110111110110110010011 & 0.52  & 0.50  \\
    5     & 100100111101010011110001110101011111011100110001000011 & 0.44  & 0.51  \\
    10    & 100001001011110110001000110011011101110111110111000011 & 0.44  & 0.50  \\
    15    & 111001000110011110111000111001011011111011011110111001 & 0.37  & 0.57  \\
    20    & 111111001111111111000111001111001011101100111111110111 & 0.26  & 0.67  \\
    25    & 101111111111111100111111110000010001101111111100111111 & 0.26  & 0.64  \\
    30    & 011110011111111111111110001111111111111001111111101111 & 0.17  & 0.85  \\
    35    & 111111110001111111111011111111111111111111111111111111 & 0.07  & 0.91  \\
    40    & 101111111111111111111111111111111111111111111111111111 & 0.02  & 0.96  \\
    45    & 111111111111111110111111111111111111111111111111111111 & 0.02  & 0.98  \\
    50    & 011111111111111111111000111111111111111111111111111111 & 0.07  & 0.98  \\
    55    & 111111111110011111111111111111111111111111111111111111 & 0.04  & 0.98  \\
    60    & 111111111111111111111111111111111111111111111111111111 & 0.00  & 0.98  \\
    65    & 111111111111111111111111111111111111111111111111111111 & 0.00  & 0.99  \\
    70    & 111111111111111110011111111111111111111111111111111111 & 0.04  & 0.96  \\
    75    & 011111111111111111111111111111111111111111111111111111 & 0.02  & 0.99  \\
    80    & 111111111111111111111111111111111111111111111111111111 & 0.00  & 1.00  \\
    85    & 111111111111111111111111111111111111111111111111111111 & 0.00  & 1.00  \\
    90    & 111111111111111111111111111111111111111111111111111111 & 0.00  & 1.00  \\
    95    & 111111111111111111111111111111111111111111111111111111 & 0.00  & 0.98  \\
    100   & 111111111111111111111111111111111111111111111111111111 & 0.00  & 0.99  \\
    105   & 111111111111111111111111111111111111111111111111111111 & 0.00  & 1.00  \\
    110   & 111111111111111111111111111111111111111111111111111111 & 0.00  & 1.00  \\
    115   & 111111111111111111111111111111111111111111111111111111 & 0.00  & 1.00  \\
    120   & 111111111111111111111111111111111111111111111111111111 & 0.00  & 1.00  \\
    125   & 111111111111111111111111111111111111111111111111111111 & 0.00  & 1.00  \\
    130   & 111111111111111111111111111111111111111111111111111111 & 0.00  & 1.00  \\
    135   & 111111111111111111111111111111111111111111111111111111 & 0.00  & 0.99  \\
    \bottomrule
    \end{tabular}%
    \end{adjustbox}

  \label{tab:rnn}%
\end{table}%

\section{Analysis}
\label{sec:analysis}
In this section, we demonstrate that the found self-attention subnetwork has the capacity of capturing the discriminative features as the full network and transferring to the downstream tasks. Besides, we compare the training time and the searching time to show that our search time is acceptable. 





\subsection{Capturing Discriminative Features}
To study the ability of Share-EAN in capturing and exploiting features of a given target, we apply Grad-CAM~\cite{selvaraju2017grad} to compare the regions where different models localize with respect to their target prediction. Grad-CAM is a technique to generate the heatmap highlighting network attention by the gradient related to the given target. Fig.~\ref{fig:cam} shows the visualization results and the softmax scores for the target with original ResNet50, Full-Share, and Share-EAN on the validation set of ImageNet2012. SE module is used in this part. The red region indicates an essential place for a network to obtain a target score while the blue region is the opposite. The results show that Share-EAN can extract similar features as Full-Share, and in some cases, Share-EAN can even capture much more details of the target associating with higher confidence for its prediction. This implies that the searched connection scheme may have a more vital ability to emphasize the more discriminative features for each class than the two baselines ~(original ResNet and Full-Share). Therefore it is reasonable to bring additional improvement on the final classification performance with Share-EAN in that the discrimination is crucial for the classification task, which is also validated from ImageNet test results in Table~\ref{tab:all_results}.

\subsection{Comparison of Training Time and Search Time}
\label{sec:search}
For NAS, we not only need to care about whether the search method can find a neural network structure that satisfies certain conditions but also need to focus on its computational cost. Taking the experiments~($\beta = 0.5$) in Table~\ref{tab:dabiao} as example, we measure the time of these experiments in Table~\ref{tab:time}. The time for training ResNet38 and ResNet164 from scratch is 1.61h, 4.46h on a single GPU 1080Ti while our RL-based method requires only 25.4\%, 21.3\% of the train time for searching from the supernet, respectively. The search time is acceptable and worthwhile as the found ticket may be applied to the downstream tasks that will be discussed in the next part.

\begin{table}[htbp]
  \centering
  \caption{The Training time of our baseline method. ``Train Time" denotes the time of training a neural network from scratch. ``Search Time" denotes the time of our RL search. ``Search/Train'' represents the ratio of search time to the train time. All experiments are conducted on a single GPU.}
    \begin{tabular}{ccccc}
    \toprule
    GPU  & Model & Train Time    & Search Time    & Search/Train \\
    \midrule
    1080Ti   & ResNet38 & 1.61 hrs & 0.41 hrs & 25.40\% \\
    1080Ti   & ResNet164 & 4.46 hrs & 0.94 hrs & 21.30\% \\
    \bottomrule
    \end{tabular}%
  \label{tab:time}%
\end{table}%

\subsection{Transferring Connection Schemes}
\label{sec:transfer}

In this part, we study the transferability of the network architecture searched by our baseline method. Specifically, we conduct experiments on transferring the optimal architecture from image classification to crowd counting task~\cite{zhang2016single,cao2018scale,li2018csrnet,hossain2019crowd} and segmentation~\cite{Everingham10}.
The model trained with classification is typically used to initialize the model for downstream tasks~\cite{ma2020mri,fan2022novel}. 
If the found network have the transferability, we will have the advantages as follows: (I) We do not need to spend extra time searching for a ticket for the new task;
(II) The model for the new task inherits a good representation ability of the pretrained model; (III) The model with fewer SAMs has less forward and back-propagation cost compared with Full-SAM. In this case, the computational cost of our RL-based search shown in Table~\ref{tab:time} is acceptable.


\textbf{Crowd counting.} Crowd counting aims to estimate the density map and predict the total number of people for a given image, whose efficiency is also crucial for many real-world applications, \textit{e.g.}, video surveillance and crowd analysis. However, most state-of-the-art works still rely on the heavy pre-trained backbone networks~\cite{liu2020efficient} for obtaining satisfactory performance on such dense regression problems. The experiments show that the Share-EAN trained on ImageNet serves as an efficient backbone network and can extract the representative features for crowd counting. 
We evaluate the transferring performance on the commonly-used Shanghai Tech dataset~\cite{zhang2016single}, which includes two parts. Shanghai Tech part A~(SHHA) has 482 images with 241,677 people counting, and Shanghai Tech part B~(SHHB) contains 716 images with 88,488 people counting. Following the previous works, SHHA and SHHB are split into train/validation/test set with 270/30/182 and 360/40/316 images, respectively. The performance on the test set is reported using the standard Mean Square Error~(MSE) and Mean Absolute Error~(MAE), as shown in Table~\ref{tab:crowdcountin}. 
Our Share-EANs outperform the baseline (Full-SA and Full-Share) while reducing the inference time increment by up to 28\% compared with the baseline.
\begin{table*}[htbp]
  \caption{Comparison of performance between different pre-trained models on crowd counting. Smaller MAE/MSE is better. }
  \centering
  \begin{adjustbox}{width=0.8\textwidth,center}
    \begin{tabular}{clcccccccl}
    \toprule
    \multirow{2}[4]{*}{Dataset} & \multicolumn{1}{c}{\multirow{2}[4]{*}{Model}} &       & \multicolumn{3}{c}{MAE/MSE~($\downarrow$)} &       & \multicolumn{3}{c}{Relative Inference Time Increment (\%)} \\
\cmidrule{4-6}\cmidrule{8-10}          &       &       & Full-SA & Full-Share & Share-EAN   &       & Full-SA & Full-Share & Share-EAN \\
    \midrule
    \multirow{2}[2]{*}{SHHB} & SE~\cite{hu2018squeeze} &       & 9.5/15.93 & 8.9/14.6 & 8.6/14.7 &       & 19.19  & 19.19  & \textbf{6.16}~({\color{black}{$\downarrow$ 13.03}}) \\
          & DIA~\cite{huang2020dianet} &       & -     & 9.1/14.9 & 8.2/13.9 &       & -     & 16.93  & \textbf{8.71}~({\color{black}{$\downarrow$ 8.22}}) \\
    \midrule
    \multirow{3}[2]{*}{SHHA} & SGE~\cite{li2019spatial} &       & 93.9/144.5 & 91.6/143.1 & 88.4/140.0 &       & 58.98  & 58.85  & \textbf{30.55}~({\color{black}{$\downarrow$ 28.30}}) \\
          & SE~\cite{hu2018squeeze} &       & 89.9/140.2 & 89.9/140.2 & 79.4/127.7 &       & 49.50  & 49.00  & \textbf{21.07}~({\color{black}{$\downarrow$ 27.93}}) \\
          & DIA~\cite{huang2020dianet} &       & -     & 92.5/130.4 & 90.3/141.6 &       & -     & 51.75  & \textbf{29.43}~({\color{black}{$\downarrow$ 22.32}}) \\
    \bottomrule
    \end{tabular}%
    \end{adjustbox}
\label{tab:crowdcountin}%
\end{table*}%
\begin{figure}[t]
    \centering
    \includegraphics[width=\linewidth]{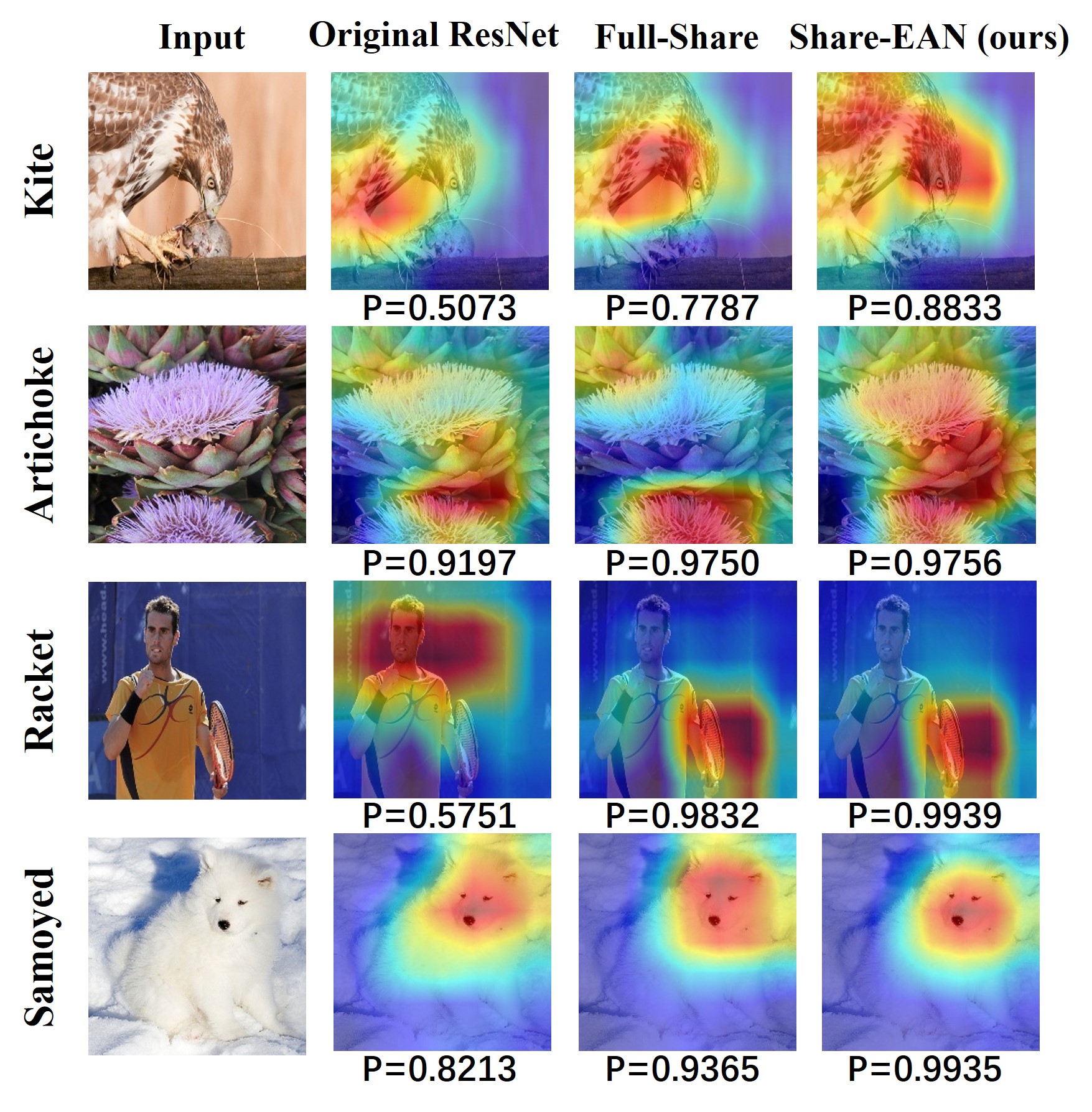}
    \caption{Grad-CAM visualization of different networks. The red region indicates an essential place for a network to obtain a target score ($\mathbf{P}$) while the blue one is the opposite.}
    \label{fig:cam}
\end{figure}

\textbf{Semantic segmentation.}
We verify the transferability of the Share-EAN on semantic segmentation task in Pascal VOC 2012~\cite{everingham2015pascal} dataset. Table~\ref{tab:semseg} shows the performance comparison of the backbone with different types of SAM, e.g., DIA and SE. Again, our results indicate that the Share-EAN can maintain the performance of the Full-Share network and significantly reduce the time increment compared with the Full-Share network, which shows Share-EAN has the capacity of transferring to segmentation.
\begin{table}[htbp]
  \caption{Performance and relative inference time increment comparison on Pascal VOC 2012 validation set.}
  \centering
    \begin{adjustbox}{width=0.45\textwidth,center}
    \begin{tabular}{ccc}
    \toprule
    Model & mIoU/mAcc/allAcc~(\%) & Time Increment~(\%)\\
    \midrule
    Original ResNet & 69.39 / 78.87 / 92.97&- \\
    Full-Share-SE & 73.03 / 82.13 / 93.74&48.16 \\
    Share-EAN-SE & 73.68 / 83.08 / 93.79& \textbf{16.43}~({\color{black}{$\downarrow$ 31.73}}) \\
    Full-Share-DIA & 74.02 / 83.11 / 93.92&64.86 \\
    Share-EAN-DIA & 73.91 / 82.93 / 93.92 & \textbf{7.68}~({\color{black}{$\downarrow$ 57.18}})\\
    \bottomrule
    \end{tabular}%
    \end{adjustbox}

  \label{tab:semseg}%
  \vspace{-0.3cm}
\end{table}%


\section{Related Works}
\label{sec:related}
\textbf{Lottery Tickets Hypothesis~(LTH).} The Lottery Ticket Hypothesis~\cite{frankle2018lottery} conjectures that: every random initialized and dense NN contains a subnetwork that can be trained in isolation with the original initialization to achieve comparable performance to the original NN. 
This original LTH attracts many researchers to rethink the training of the overparameterized model, leading to variants of LTH under different learning paradigms and  machine learning fields~\cite{Ramanujan_2020_CVPR,NEURIPS2020_b6af2c97,chen2021earlybert}

On the other hand, Malach et al.~\cite{pmlr-v119-malach20a} try to prove a LTH variants~\cite{Ramanujan_2020_CVPR} by showing that given a target NN of depth $l$ and width $d$, any random initialized network with depth $2l$ and width $O\left(d^{5} l^{2} / \epsilon^{2}\right)$ contains subnetworks that can approximate the target network with $\epsilon$ error. Following works~\cite{NEURIPS2020_1e949147,pensia2020optimal} further reduce the width to $O(d \log (d l / \epsilon))$ for the random initialized network. 

\noindent\textbf{Neural Architecture Search~(NAS).} Designing a satisfactory neural architecture automatically, also known as neural architecture search, is of significant interest for academics and industries. Such a problem may always be formulated as searching for the optimal combination of different network granularities. The early NAS works require expensive computational costs for scratch-training a massive number of architecture candidates~\cite{zoph2016neural, zoph2018learning}. To alleviate the searching cost, the recent advances of one-shot approaches for NAS bring up the concept of supernet based on the weight-sharing heuristic. Supernet serves as the search space embodiment of the candidate architectures, and it is trained by optimizing different sub-networks from the sampling paths, \textit{e.g.}, SPOS~\cite{guo2020single}, GreedyNAS~\cite{you2020greedynas}.



\noindent\textbf{Self-Attention Mechanism.} The self-attention mechanism is widely used in CNNs for computer vision~\cite{hu2018squeeze,wang2018non,huang2020dianet,cao2019gcnet,li2019spatial,liang2020instance}. 
Squeeze-Excitation~(SE) module~\cite{hu2018squeeze} leverages global average pooling to extract the channel-wise statistics and learns the non-mutually-exclusive relationship between channels. Spatial Group-wise Enhance~(SGE) module~\cite{li2019spatial} learns to recalibrate features by saliency factors learned from different groups of the feature maps. Dense-Implicit-Attention~(DIA) module~\cite{huang2020dianet} captures the layer-wise feature interrelation with a recurrent neural network~(RNN).

\section{Conclusion}
Lottery Ticket Hypothesis for Self-attention Networks is proposed in this paper, which is supported by numerical and theoretical evidence. Then, to find a ticket, we propose an effective connection scheme searching method based on policy gradient as a baseline to find a ticket. The self-attention network found by our method can maintain accuracy, reduce parameters and accelerate the inference speed. Besides, we illustrate that the found network has the capacity of capturing the informative features and transferring to other computer vision tasks.


{\appendix[Proof of Theorem 1]

\textbf{Theorem 1}. A 1-hidden-layer feed-forward NN is defined as $NN(x)=W^{2}\sigma(W^1x)$, where input $x\in \mathds{R}^{d\times1}$ with $\|x\|_2\leq 1$, $W^1$ is of size $m\times d$, $W^2$ is of size $1\times m$ and $\sigma$ is ReLU activation. $W^1_{i,j}$ is initialized i.i.d. by the Gaussian distribution $\mathcal{N}(0, {(\frac{1}{\sqrt{m}})}^{2})$, and $W^2_{1,j}$ is initialized by the uniform distribution $Uniform\{1,-1\}$. Let $\mathcal{P}(d-1,\epsilon)$ be $\mathbb{P}\{\chi^2(d-1)\geq \epsilon^2\}$, where $\chi^2(d-1)$ is a chi-square variable with $d-1$ degree of freedom. Then for any $\epsilon, \delta>0$, when the number of hidden neurons $m>\frac{\ln(\delta)}{\ln(\mathcal{P}(d-1,\epsilon))}$, then there exists the row $j$ of $W^1$ such that when we set the row $j$ to be zero, i.e., $B_jW^1$ with $B_j=diag\{1,\cdots,1,0,1,\cdots,1\}$ (the $j ^{\rm th}$ entry is 0), we have 
\begin{align*}
    \|W^{2}\sigma(W^1x)-W^{2}\sigma(B_jW^1x)\|<\epsilon,
\end{align*}
with probability higher than $1-\delta$.

\begin{proof}
Since $W^2_{1,j}\sim Uniform\{-1,1\}$, we have $\|W_2\|_2=\sqrt{m}$. We denote the row $s$ of $W^1$ as $W^1_{s:}$. Let's consider the following probability,
\begin{align*}
    \mathbb{P}\{\nexists s: \|W^1_{s:}\|_2< \frac{\epsilon}{\sqrt{m}}\}&=\mathbb{P}\{\cup_{s=1}^m \|W^1_{s:}\|_2\geq \frac{\epsilon}{\sqrt{m}}\}
    \\&=\prod_{s=1}^m\mathbb{P}\{\|W^1_{s:}\|_2\geq \frac{\epsilon}{\sqrt{m}}\}
    \\&=\prod_{s=1}^m\mathbb{P}\{\frac{1}{m}\chi^2(d-1)\geq \frac{\epsilon^2}{m}\}
    \\&=\mathcal{P}(d-1,\epsilon)^m.
\end{align*}
Let $\mathcal{P}(d-1,\epsilon)^m<\epsilon$, and then we have $m>\frac{\ln(\delta)}{\ln(\mathcal{P}(d-1,\epsilon))}$. Therefore, when $m>\frac{\ln(\delta)}{\ln(\mathcal{P}(d-1,\epsilon))}$, with probability greater than $1-\delta$, there exists $j$ such that $\|W^1_{j:}\|_2< \frac{\epsilon}{\sqrt{m}}$. 

Let $G=diag\{W^1x\geq 0\}$, where the $s$-th diagonal component of $G$ is 1 if $W^1_{s:}x\geq 0$ else 0. Then $\sigma(W^1x)=GW^1x$ and $\sigma(B_jW^1x)=GB_jW^1x$. Finally, we obtain
\begin{align*}
    &\|W^{2}\sigma(W^1x)-W^{2}\sigma(B_jW^1x)\|_2
    \\=&\|W^{2}GW^1x-W^{2}GB_jW^1x\|_2
    \\\leq &\|W^{2}\|_2\|G\|_2\|W^1-B_jW^1\|_2\|x\|_2
    \\\leq &\sqrt{m}\times 1\times \|W^1_{j:}\|_2\times 1\leq \epsilon,
\end{align*}
with probability greater than $1-\delta$.
\end{proof}

 }
 
{\appendix[Proof of Theorem 2]

\textbf{Theorem 2}. Let $T(x)$ be a Lipschitz continuous and Lebesgue integrable function in $d$-dimensional compact set $K$. And $R_{\text{full}}(x,\theta_{\text{full}})$ is a ReLU ResNet structure with parameters $\theta_{\text{full}}$. Let $\epsilon_0>0$ be a fixed constant. Suppose that there exists $\theta_{\text{full}}^0$ such that
$
    \int_{K}|R_{\text{full}}(x,\theta_{\text{full}}^0) - T|dx \leq \frac{\epsilon_0}{2}.
$
If the width of each layer in $R_{\text{full}}(x,\theta_{\text{full}})$ is larger than $d$ and the depth of $R_{\text{full}}(x,\theta_{\text{full}})$ is larger than a constant that depends on $\epsilon_0$, then for any $\epsilon \in (\epsilon_0, 1)$, there exists a subnetwork $R_\text{sub}(x)$ of $R_{\text{full}}(x,\theta_{\text{full}})$ such that 
\begin{equation}
    \int_{K}|R_{\text{full}}(x,\theta_{\text{full}}^0) - R_\text{sub}(x)|dx \leq \epsilon.
\end{equation}

\begin{lemma}\cite{lin2018resnet} For any $d\in \mathbb{N}$, the family of ResNet with one-neuron hidden layers and ReLU activation function can universally approximate any Lebesgue integrable function $f$. In other words, for any $\epsilon > 0$, there is a ResNet $R$ with finitely many layers and width not larger than $d$ such that 
\begin{equation}
    \int_{\mathds{R}^d}|f(x) - R(x)|dx \leq \epsilon.
\end{equation}
\end{lemma}

\begin{proof}
See \cite{lin2018resnet}.
\end{proof}

We use the notation ${\rm dep}(\cdot)$ to denote the depth of a network. 

\begin{lemma}\label{lema:ext}(Extension strategy) Let $T$ be a Lebesgue integrable $d$-dimensional function. For an $\epsilon > 0$, a ReLU ResNet $g(x,\theta_g^0)$ with parameters $\theta_g^0$ satisfies $\int_{\mathds{R}^{d}}\left|g\left(x, \theta_{g}^{0}\right)-T\right| d x \leq \epsilon$, then there exists $f(x, \theta_{f}^{0})$ with ${\rm dep}(f)>{\rm dep}(g)$ such that
\begin{equation}
    \int_{\mathds{R}^{d}}\left|f\left(x, \theta_{f}^{0}\right)-T\right| d x \leq \epsilon.
\end{equation}
Here $f(x, \theta_{f}^0)$ is obtained by adding some layers to the last layer of $g(x, \theta_{g}^0)$.
\end{lemma}
 \begin{figure}[H]
    \centering
    \includegraphics[width=1\linewidth]{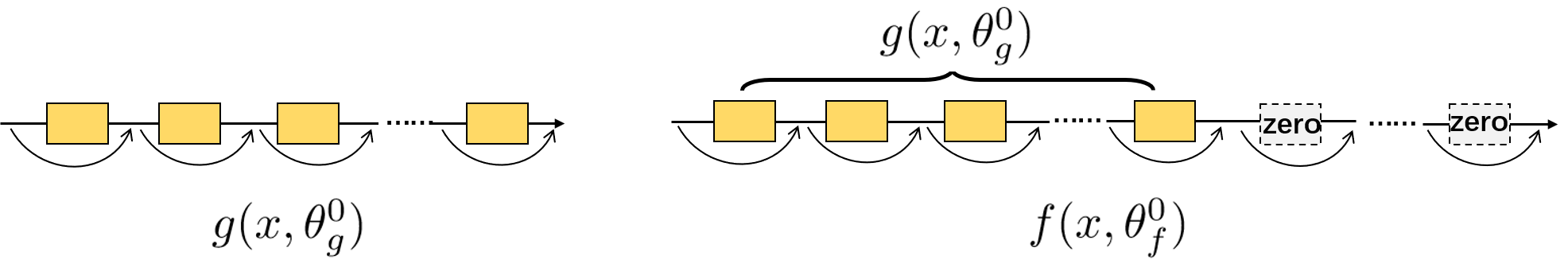}
    \caption{The structure $f(x, \theta_{f}^0)$ and $g(x, \theta_{g}^0)$.}
    \label{fig:fg}
\end{figure}
\begin{proof}
As shown in Fig.~~\ref{fig:fg}, we can expand $g$ by adding some skip connection layers to its last layer. And then we set all the values of the parameters of the extra layers to zeros. Through the skip connections, we have $g(x, \theta_{g}^0) = f(x, \theta_{f}^0)$ and $\text{dep}(g) < \text{dep}(f)$. Therefore,
\begin{equation}
    \int_{\mathds{R}^{d}}\left|f\left(x, \theta_{f}^{0}\right)-T\right| d x  = \int_{\mathds{R}^{d}}\left|g\left(x, \theta_{g}^{0}\right)-T\right| d x \leq \epsilon.
\end{equation}
\end{proof}

\begin{lemma} \label{lema:depth}Let $f$ be the Lebesgue integrable function defined on $d$-dimensional compact set $K\subset \mathds{R}^d$. $\forall \epsilon > 0$, according to Lemma 1, there is a ResNet $R(x)$ with finitely many layers and width not larger than $d$ such that $\int_{K}|f(x) - R(x)|dx \leq \epsilon$. Then the depth of $R(x)$ is $O(1/r^d)$, where $r$ satisfies $\omega_K(r) \leq \epsilon/\text{Vol}(K)$ with $\omega_K(r)$ defined by
\begin{equation}
    \omega_K(r) = \max_{x,y\in K,||x-y||\leq r}|f(x) - f(y)|.
\end{equation}

\end{lemma}

\begin{proof}
Refer to \cite{lin2018resnet}.
\end{proof}

Now, we prove \textbf{Theorem 2} through the lemmas above. 
\begin{proof}
First, $T(x)$ is a Lipschitz continuous function, so
\begin{equation}
    |T(x) - T(y)| \leq L|x - y|,
    \label{lpxz}
\end{equation}
for $x,y \in K$, where $L$ is a constant. Then we have
\begin{align*}
\omega_K(r) &= \max_{x,y\in K,||x-y||\leq r}|T(x) - T(y)|	\\
&	\leq \max_{x,y\in K,||x-y||\leq r}L|x - y|\tag*{Since Eq.(\ref{lpxz})}	\\
&	\leq Lr.	
\end{align*}
Let $\omega_K(r) \leq Lr = \epsilon/\text{Vol}(K)$, and then we have
\begin{equation}
    r = \frac{\epsilon}{\text{Vol}(K)\cdot L}.
\end{equation}
When $r = \epsilon/(\text{Vol}(K)\cdot L)$, then for any $\epsilon \in (\epsilon_0,1)$,
\begin{equation}
    O(1/r^d) = O((\frac{L}{\epsilon})^d) < O((\frac{L}{\epsilon_0})^d)=C(\frac{L}{\epsilon_0})^d,
\end{equation}
where $C$ is a constant. 
Therefore, according to Lemma \ref{lema:depth},  $\forall \epsilon \in (\epsilon_0,1)$, there exist a ResNet $R_{\text{short}}(x)$ with width not greater than $d$ and the depth of at most $C(\frac{L}{\epsilon_0})^d$ such that $\int_{K}|T(x) - R_{\text{short}}(x)|dx \leq \epsilon/2$. When the depth of $R_{\text{full}}(x,\theta_{\text{full}})$ is greater than $C(\frac{L}{\epsilon_0})^d$, we can use Lemma \ref{lema:ext}~(extension strategy) to construct a function $R_{\text{long}}(x)$ such that 
\begin{equation}
    \text{dep}(R_{\text{long}}) = \text{dep}(R_{\text{full}}),
    \label{eq:depth}
\end{equation}
and the width of $R_{\text{long}}$ is not greater than $d$. Also, for any $x \in K$, $R_{\text{long}}(x) = R_{\text{short}}(x)$. So we have
\begin{equation}
        \int_{K}|T(x) - R_{\text{long}}(x)|dx = \int_{K}|T(x) - R_{\text{short}}(x)|dx\leq \epsilon/2.
    \label{eq:111}
\end{equation}

Then
\begin{align}
\int_{K}|R_{\text{full}}(x,\theta_{\text{full}}^0) - R_\text{long}(x)|dx &\leq \int_{K}|T(x) - R_\text{long}(x)|dx \\
&+ \int_{K}|R_{\text{full}}(x,\theta_{\text{full}}^0) - T(x)|dx	\\
&\leq 	\epsilon/2 + \epsilon_0/2	\leq \epsilon	\label{ieq:bound}
\end{align}

Note that $\text{dep}(R_{\text{long}}) = \text{dep}(R_{\text{full}})$. Also, $R_{\text{long}}$ is a ResNet with width not larger than $d$ while the width of $R_{\text{full}}$ is greater than $d$. Therefore, $R_{\text{long}}$ is a subnetwork of $R_{\text{full}}$ and satisfies the inequality~(\ref{ieq:bound}). 

\end{proof}
 }


{\appendix[Different types of SAMs]
In this part, we review the SAMs used in our paper, \textit{i.e.}, SE~\cite{hu2018squeeze}, SGE~\cite{li2019spatial} and DIA~\cite{huang2020dianet}. We follow some notations of Section~\ref{sec:prelim}. Let $x_\ell$ be the input of the $\ell^\text{th}$ block, $f_\ell(\cdot)$ be the residual mapping, and $M(\cdot; W_\ell)$ be the SAM in the $\ell^\text{th}$ block with the parameters $W_\ell$. The attention is formulated as $M(f_\ell(x_\ell);W_\ell)$. We denote $f_\ell(x_\ell)$ as $X^{(\ell)}$ of size $C\times H\times W$, where $C,H$ and $W$ denote channel, height and width, respectively. For simplicity, we denote $X_{chw}^\ell=X^\ell[c,h,w]$ as the value of pixel $(h,w)$ at the channel $c$ and $X_{c}^\ell = X^\ell[c,:,:]$ as the tensor at the channel $c$.

\textbf{SE Module.} SE module utilizes average pooling to extract the features and processes the extracted features by a one-hidden-layer fully connected network.  

First, the SE module squeezes the information of channels by the average pooling,
\begin{align}
  m_{c}^\ell = \text{AVG}(X_{c}^\ell) = \frac{1}{H\cdot W}\sum_{h=1}^H\sum_{w=1}^W X_{chw}^\ell,
  \label{eqn:avgpool}
\end{align}
where $c=1,\cdots,C$. Then, an one-hidden-layer fully connected network $\text{FC}(\cdot;W_\ell)$ with ReLU activation is used to fuse the information of all the channels and here $W_\ell$ is the parameter. The hidden layer node size is $C//r$, where ``$//$'' is exact division and ``r'' denotes reduction rate. The reduction rate is 16 in our experiments. Finally, a sigmoid function (\textit{i.e.}, $\text{sig}(z) = 1/(1+e^{-z})$) is applied to the processed features and we get the attention as follows,
\begin{align}
    [\delta_{1}; \cdots; \delta_{C}] = \text{sig}(\text{FC}([m_{1}^\ell; \cdots; m_{C}^\ell];W_\ell)).
    \label{eqn:se_attention}
\end{align}

\textbf{DIA Module.} DIA module integrates the block-wise information by an LSTM~(Long Short-Term Memory). Let $m_{c}^{\ell}$ be the output of average pooling as Eq.\ref{eqn:avgpool}.
Then $m_{c}^{\ell}$ is passed to LSTM along with a hidden state vector $h_{\ell-1}$ and a cell state vector $c_{\ell-1}$, where $h_0$ and $c_0$ are initialized as zero vectors. The LSTM generates $h_{\ell}$ and $c_{\ell}$ at the $\ell^\text{th}$ block, \textit{i.e.,}
    \begin{align}
    (h_\ell,c_\ell) = \text{LSTM}([m_{1}^\ell; \cdots; m_{C}^\ell], h_{\ell-1},c_{\ell-1};W),
    \label{eqn:lstm-dia}
    \end{align}
     where $W$ is the trainable parameter of the LSTM. The hidden state vector $h_t$ is used as attention to recalibrate feature maps. The reduction ratio within LSTM introduced in ~\cite{huang2020dianet} is 4 for CIFAR100 or 20 for ImageNet2012. 

\textbf{SGE Module.} SGE divides the feature maps into different groups and then utilizes the global information from the group to recalibrate its features. Let G be the number of groups and then each group has $C//G$ feature maps. Denote $Y^\ell$ of size $(C//G)\times H\times W$ as a group of feature maps within $X^\ell$. The extracted feature for the group $Y^\ell$ is 
\begin{align}
    g_{c}^\ell = \text{AVG}(Y_{c}^\ell) = \frac{1}{H\cdot W}\sum_{h=1}^H\sum_{w=1}^W Y_{chw}^\ell.
\end{align}
Let $g$ be $[g_{1}^\ell;\cdots;g_{C//G}^\ell]$. The importance coefficient for each pixel $(h,w)$ is defined as 
\begin{align}
    p_{hw} = g\cdot Y[:,h,w],
\end{align}
where $\cdot$ is dot product. Then $p_{hw}$ is normalized by
\begin{align}
    \hat{p}_{hw} = \frac{p_{hw}-\mu}{\sigma+\epsilon},
\end{align}
where the mean $\mu$ and variance $\sigma^2$ are defined by 
\begin{align}
    \mu = \frac{1}{H W}\sum_{h=1}^H\sum_{w=1}^Wp_{hw}, \quad \sigma^2 =  \frac{1}{H W}\sum_{h=1}^H\sum_{w=1}^W(p_{hw}-\mu)^2.
\end{align}
An additional pair of parameters $(\gamma, \beta)$ are introduced for the group $Y^\ell$ to rescale and shift the normalized features, and SGE modules get the attention for $Y[:,h,w]$ as follows,
\begin{align}
    \text{sig}(\gamma\hat{p}_{hw}+\beta).
\end{align}
$G$ is 4 for CIFAR100 and 64 for ImageNet2012 experiments.

}

{\appendix[Training details for supernet]
In Alg.\ref{alg:ean}, we have presented the training strategy for supernet briefly. We show this process in an intuitive way in Fig.~\ref{fig:algo}. First, for each step $t$, we can sample one connection scheme from a $[Bernoulli(\beta)]^m$ distribution. Next, based on this sampled connection scheme, we can obtain a subnetwork from supernet. Then, we train this subnetwork on the training set $D_\text{train}$.

 \begin{figure}[H]
    \centering
    \includegraphics[width=1\linewidth]{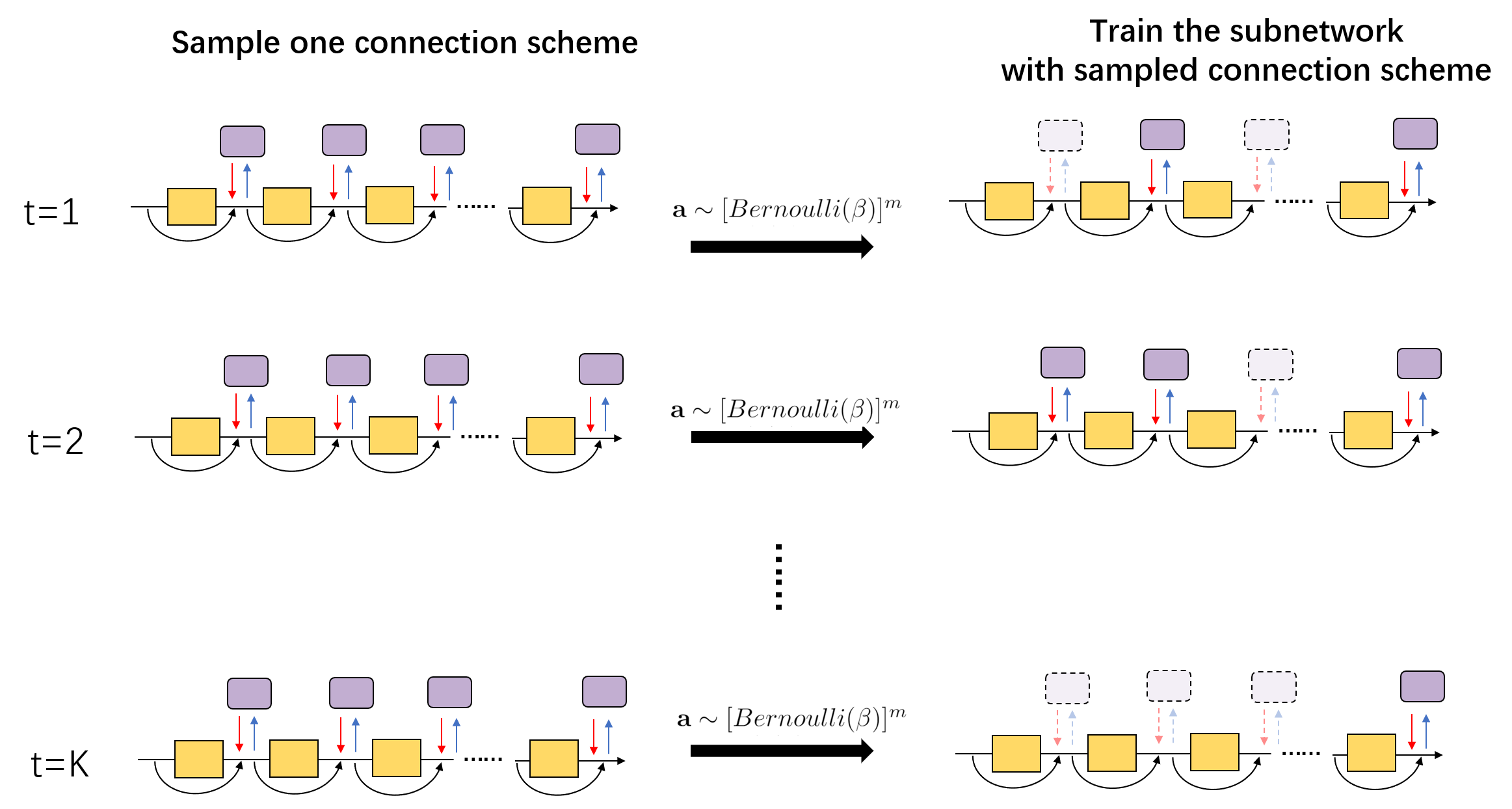}
    \caption{Procedure of training a supernet.}
    \label{fig:algo}
\end{figure}

}

{\appendix[Training details for Controller]
In this part, we provide the training details for the controller. The training process is shown in Fig.~\ref{fig:arch}. The reward function of the connection scheme consists of three parts, \textit{i.e.}, sparsity reward, validation reward, and curiosity bonus. Besides, we supplement some details of Alg.~\ref{alg:ean}. 

\textbf{Sparsity Reward} $g_\text{spa}$. One of our goals is to accelerate the inference of the Full-SA network. To achieve it, we complement a sparsity reward $g_\text{spa}$ to encourage the controller to generate the schemes with fewer connections between SAMs and backbone. We define $g_\text{spa}$ by 
\begin{align}
    g_\text{spa} =  1 - \frac{\left\Vert \mathbf{a}\right\Vert_0}{m},
    \label{eqn:sparse}
\end{align}where $\left\Vert\cdot\right\Vert_0$ is a zero norm that counts the number of non-zero entities, and $m$ is the number of blocks.

\textbf{Validation Reward} $g_\text{val}$. Another goal is to find the schemes with which the networks can maintain the original accuracy. Hence, we use the validation accuracy of the subnetwork $\Omega(\mathbf{x}|\mathbf{a})$ sampled from the supernet as a reward, which depicts the performance of its structure. The accuracy of $\Omega(\mathbf{x}|\mathbf{a})$ on $D_\text{val}$ is denoted as $g_\text{val}$. In fact, it is popular to use validation accuracy of a candidate network as a reward signal in NAS~\cite{pham2018efficient,zoph2016neural,guo2020single,zoph2018learning,you2020greedynas}. Furthermore, it has been empirically proven that the validation performance of the subnetworks sampled from a supernet can be positively correlated to their stand-alone performance~\cite{bender2018understanding}. We evaluate the correlation between the validation accuracy of subnetworks sampled from a supernet and their stand-alone performance on CIFAR100 with ResNet and SE module over 42 samples and obtain the Pearson coefficient is 0.71, which again confirms the strong correlation as shown in the previous works.


\textbf{Curiosity Bonus} $g_\text{rnd}$. To encourage the controller to explore more potentially useful connection schemes, we add the Random Network Distillation~(RND) curiosity bonus~\cite{burda2018exploration} in our reward. 
Two extra networks with input $\mathbf{a}$ are involved in the RND process, including a target network $\sigma_1(\cdot)$ and a predictor network $\sigma_2(\cdot;\phi)$, where $\phi$ is the parameter set. The parameters of $\sigma_1(\cdot)$ are randomly initialized and fixed after initialization, while $\sigma_2(\cdot;\phi)$ is trained with the connection schemes collected by the controller. 

The basic idea of RND is to minimize the difference between the outputs of these two networks, which is denoted by term $\sigma_\phi(\cdot) = \left\Vert\sigma_1(\cdot)-\sigma_2(\cdot;\phi)\right\Vert_2^2$, over the seen connection schemes.
If the controller generates a new scheme $\mathbf{a}$, $\sigma_\phi(\mathbf{a})$ is expected to be larger because the predictor $\sigma_2(\cdot;\phi)$ never trains on scheme $\mathbf{a}$. Then, we denote the term $\left\Vert\sigma_1(\mathbf{a})-\sigma_2(\mathbf{a};\phi)\right\Vert_2^2$ as $g_\text{rnd}$, which is used as curiosity bonus to reward the controller for exploring a new scheme. Besides, in Fig.~\ref{fig:enas_vs_ean}, we empirically show that RND bonus mitigates the fast convergence of early training iterations, leading to exploration for more schemes.

To sum up, our reward $G(\mathbf{a})$ becomes
\begin{align}
    G(\mathbf{a}) = \lambda_1\cdot g_\text{spa} + \lambda_2 \cdot
    g_\text{val}+ \lambda_3 \cdot g_\text{rnd}, 
    \label{eqn:rnd_reward}
\end{align}
where $\lambda_1, \lambda_2, \lambda_3$ are the coefficients for each reward. 

\textbf{Data Reuse.} To improve the utilization efficiency of sampled connection schemes and speed up the training of the controller, we incorporate Proximal Policy Optimization~(PPO)~\cite{schulman2017proximal} in our method. As shown in Alg.~\ref{alg:ean}, after the update of parameter $\theta$ and $\phi$, we put the tuple $(\mathbf{p}_\theta,~\mathbf{a},~G(\mathbf{a}))$ into a buffer. At the later step, we retrieve some used connection schemes and update $\theta$ as follows: 
\begin{align}
\begin{split}
    \kappa&=\mathbb{E}_{\mathbf{a}\sim\mathbf{p}_{\theta_{old}}}\left[G(\mathbf{a})\sum_{i=1}^m\frac{\hat{p}^i_\theta}{\hat{p}^i_{\theta_{old}}}\nabla_\theta\log \hat{p}_\theta^i\right],\\
    \theta&\gets\theta+\eta\cdot\kappa,
\end{split}
    \label{eqn:ppo}
\end{align}
 where $\eta$ is the earning rate and the $\theta_{old}$ denotes the $\theta$ sampled from buffer. Note that we do not update with PPO until the controller is updated for $h$ times. 
\begin{figure}[t]
\begin{center}
\includegraphics[width=0.4\textwidth]{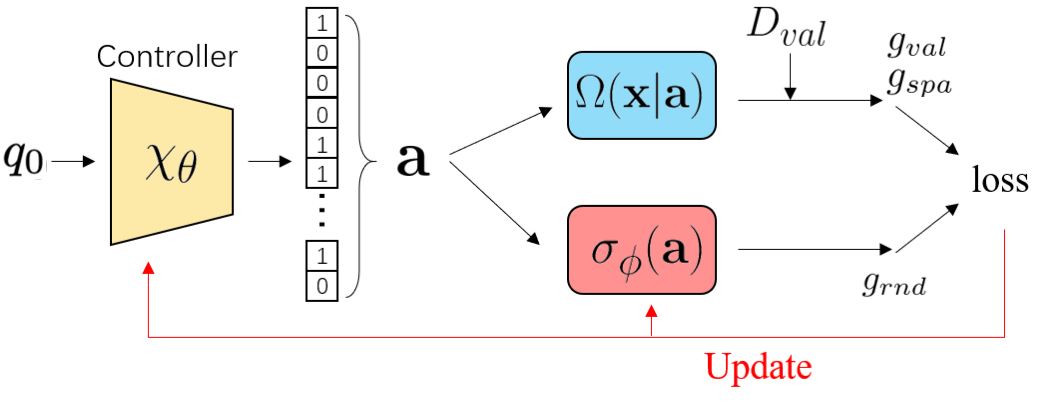}
\end{center}
\caption{The illustration of our policy-gradient-based method to search an optimal scheme.}
\label{fig:arch}
\end{figure}

Finally, we give hyper-parameter settings for training a controller on different datasets. 

\textbf{CIFAR100.} We optimize the controller for 1000 iterations with momentum SGD. The learning rate is set to be 5$\times 10^{-2}$. The time step $h$ to apply PPO is 10. 

\textbf{ImageNet2012.} We optimize the controller for 300 iterations with momentum SGD. The learning rate is set to be 5$\times 10^{-2}$. The time step $h$ to apply PPO is 10.

}

{\appendix[Training details for Stand-alone Performance]
In this part, we introduce the parameter setting for the model trained from scratch. In our experiments, we use cross-entropy loss and optimize the model by SGD with momentum 0.9 and initial learning rate 0.1. The weight decay is set to be $10^{-4}.$ The results for all search methods reported are the best out of three candidates with the highest reward (lowest validation loss for DARTS) in one search. 

\textbf{CIFAR100.} When ResNet164 is used, the model is trained for 164 epochs with the learning rate dropped by 0.1 at 81, 122 epochs. When ResNet38 is used, the model is trained for 100 epochs with the learning rate following cosine learning rate decay. In order to mitigate the over-fitting problems faced by the deep networks, ResNet164 is trained with random flipping and cropping. ResNet38 is trained with random flipping.

\textbf{ImageNet2012.} We use the ResNet50 backbone for ImageNet experiments. The network is trained for 120 epochs with the learning rate dropped by 0.1 at every 30 epochs.

}

 

\bibliographystyle{IEEEtran}
\bibliography{ref.bib}

\vspace{11pt}

\vfill

\end{document}